\documentclass[10pt,a4paper]{article}

\usepackage{times}
\usepackage{epsfig}
\usepackage{bm,array}
\usepackage{graphicx}
\usepackage{amsmath}
\usepackage{amssymb}
\usepackage{amsthm}
\usepackage{authblk}
\usepackage[top=.5in, bottom=.8in, left=1in, right=.8in]{geometry}

\newcommand{\one}{\mathbf{1}}

\newcommand{\bx}{{\bf{x}}}

\newcommand{\by}{{\bf{y}}}

\newcommand{\bv}{{\bf{v}}}
\newcommand{\bd}{{\bf{d}}}

\newcommand{\bw}{{\bf{w}}}

\newcommand{\ie}{\textit{i.e.}~}
\newcommand{\eg}{\textit{e.g.}~}

\DeclareMathOperator*{\argmin}{arg\,min}

\DeclareMathOperator*{\diag}{diag}
\DeclareMathOperator*{\vect}{vec}

\DeclareMathOperator*{\Tr}{Tr}
\DeclareMathOperator{\rank}{rank}

\newcommand{\cf}{cf}
\newcommand{\fun}{g}
\newcommand{\fundev}{g'}

\newcommand{\reals}{\mathbb{R}}

\newcommand{\FIX}[1]{{\color{black}#1}}
\newtheorem{lemma}{Lemma}
\newtheorem{prop}{Proposition}

\newtheorem{Corollary}{Corollary}

\newcolumntype{C}{>{\centering\arraybackslash}p{4.8em}}

\usepackage[pagebackref=true,breaklinks=true,colorlinks,bookmarks=false]{hyperref}

\begin{document}

\title{Training Deep Networks with Structured Layers by \\Matrix Backpropagation\thanks{\small This is an extended version of the ICCV 2015 article \cite{ionescu15iccv}}}

\author[2,3]{Catalin Ionescu\thanks{\tt\small catalin.ionescu@ins.uni-bonn.de}}
\author[3]{Orestis Vantzos\thanks{\tt\small orestis.vantzos@ins.uni-bonn.de}}
\author[1,2]{Cristian Sminchisescu\thanks{\tt\small cristian.sminchisescu@math.lth.se}}
\affil[1]{Department of Mathematics, Faculty of Engineering, Lund University}
\affil[2]{Institute of Mathematics of the Romanian Academy}
\affil[3]{Institute for Numerical Simulation, University of Bonn}


\maketitle

\begin{abstract}
Deep neural network architectures have recently produced excellent results in a variety of areas in artificial intelligence and visual recognition, well surpassing traditional shallow architectures trained using hand-designed features. The power of deep networks stems \emph{both} from their ability to perform local computations followed by pointwise non-linearities over increasingly larger receptive fields, \emph{and} from the simplicity and scalability of the gradient-descent training procedure based on backpropagation. An open problem is the inclusion of layers that perform global, structured matrix computations like segmentation (e.g.~normalized cuts) or higher-order pooling (e.g.~log-tangent space metrics defined over the manifold of symmetric positive definite matrices) while preserving the validity and efficiency of an end-to-end deep training framework. In this paper we propose a sound mathematical apparatus to formally integrate global structured computation into deep computation architectures. At the heart of our methodology is the development of the theory and practice of backpropagation that generalizes to the calculus of \FIX{adjoint} matrix variations. The proposed \emph{matrix backpropagation} methodology applies broadly to a variety of problems in machine learning or computational perception. Here we illustrate it by performing visual segmentation experiments using the BSDS and MSCOCO benchmarks, where we show that deep networks relying on second-order pooling and normalized cuts layers, trained end-to-end using matrix backpropagation, outperform counterparts that do not take advantage of such global layers.
\end{abstract}

\section{Introduction}

Recently, the end-to-end learning of deep architectures using stochastic gradient descent, based on very large datasets, has produced impressive results in realistic settings, for a variety of computer vision and machine learning domains\cite{krizhevsky2012imagenet,Simonyan14c,carreira2014distributed,schmidhuber14}. There is now a renewed enthusiasm of creating integrated, automatic models that can handle the diverse tasks associated with an able perceiving system.

One of the most widely used architecture is the convolutional network (ConvNet) \cite{lecun1998gradient,krizhevsky2012imagenet}, a deep processing model based on the composition of convolution and pooling with pointwise nonlinearities for efficient classification and learning. While ConvNets are sufficiently expressive for classification tasks, a comprehensive, deep architecture, that uniformly covers the types of structured non-linearities required for other calculations has not yet been established. 
In turn, matrix factorization plays a central role in classical (shallow) algorithms for many different computer vision  and machine learning problems, such as image segmentation \cite{shi00ncuts}, feature extraction, \FIX{descriptor design} \cite{harris1988combined,carreira2012semantic}, structure from motion \cite{tomasi1992shape}, camera calibration \cite{hartley2003multiple}, and dimensionality reduction \cite{jolliffe2002principal, Belkin2003}, among others. Singular value decomposition (SVD) in particular, is extremely popular because of its ability to efficiently produce global solutions to various problems.

In this paper we propose to enrich the dictionary of \FIX{deep networks} with layer generalizations \FIX{and fundamental} matrix function computational blocks that have proved successful and flexible over years in vision and learning models with global constraints. We consider layers which are explicitly \emph{structure-aware} in the sense that they \FIX{preserve global invariants} of the underlying problem. 
Our paper makes two main mathematical contributions. The first shows how to operate with structured layers when learning a deep network. For this purpose we outline a \FIX{\emph{matrix generalization of backpropagation} that  offers a rigorous, formal treatment of global properties.} Our second contribution is to further derive and instantiate the methodology to learn convolutional networks for two different and very successful types of structured layers: 1) \emph{second-order pooling} \cite{carreira2012semantic} and 2) \emph{normalized cuts} \cite{shi00ncuts}. \FIX{An illustration of the resulting deep architecture for O$_2$P is given in fig. \ref{fig:my_label}.} In challenging datasets like BSDS and MSCOCO, we experimentally demonstrate the feasibility and added value of these two types of networks over counterparts that are not using global computational layers. \begin{figure*}[]
\centering
\includegraphics[width=\textwidth]{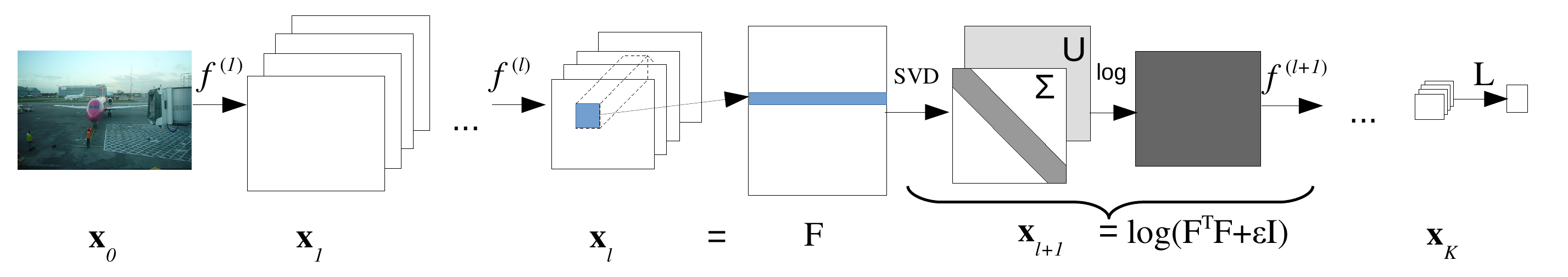}
\caption{Overview of the DeepO$_2$P recognition architecture made possible by our methodology. The levels $1\ldots l$ represent standard convolutional layers. Layer $l+1$ is the global matrix logarithm layer presented in the paper. This is followed by fully connected layers and a logistic loss. The methodology presented in the paper enables analytic computation over both local and global layers, in a system that remains trainable end-to-end, for all its local and global parameters, using matrix variation generalizations entitled \emph{matrix backpropagation}.}
\label{fig:my_label}
\end{figure*}

\section{Related Work}

Our work relates to both the extensive literature in the area of (deep) neural networks (see \cite{schmidhuber14} for a review) and with (shallow) architectures that have been proven popular and successful in machine learning and computer vision\cite{shi00ncuts,bach06,sminchisescu_ijcv04,eriksson2011normalized,carreira2012semantic}. While deep neural networks models have focused, traditionally, on generality and scalability, the shallow computer vision and machine learning architectures have often been designed with global computation and structure modeling in mind. Our objective in this work is to provide first steps and one possible approach towards formally marrying these two lines of work.

Neural networks in their modern realization can be traced back at least to \cite{mcculloch1943logical}. The Perceptron \cite{frank1957perceptron} was the first two layer network, although limited in expressiveness. The derivation of backpropagation\cite{werbos74} and its further advances more than a decade later \cite{rumelhart86,lecun1988theoretical}, allowed the development and the integration of new layers and the exploration of complex, more expressive architectures. This process lead to a successes in practical applications, \eg for digit recognition \cite{lecun1998gradient}. More recently, the availability of hardware, the large scale datasets \cite{krizhevsky2012imagenet}, and the development of complex enough architectures, lead to models that currently outperform all existing representations for challenging, general recognition problems. This recommends neural networks as one of the forefront methodologies for building representations for prediction problems in computer vision and beyond\cite{hinton12deep}. \cite{Simonyan14c} then showed that even more complex, deeper models can obtain even better results. This lead computer vision researchers to focus on transferring this success to the detection and semantic segmentation problems, fields where handcrafted  features\cite{dalal2005histograms,lazebnik2006beyond}, statistically inspired\cite{perronnin10,gosselin2014revisiting,carreira2012semantic} and deformable part models\cite{felzenszwalb2010object} were dominant at the time. R-CNN \cite{girshick2014rich} uses standard networks (e.g.~AlexNet \cite{krizhevsky2012imagenet} or VGG-16 \cite{Simonyan14c}) to classify object proposals for detection. SDS \cite{hariharan2014simultaneous} uses two input streams, one the original image and the second the image with the background of the region masked each with AlexNet architectures to take advantage of the shape information provided by the mask. He \emph{et al.} \cite{he2014spatial,dai2014convolutional} propose a global spatial pyramid pooling layer before the fully connected layers, which perform simple max-pooling over pyramid-structured cells of the image. \cite{ciresan2012multi} uses committees to improve robustness and pushed performance close to, or beyond, human performance on tasks like traffic sign recognition and house number identification. In our first application we illustrate a deep architecture with a new log-covariance pooling layer that proved dominant for free-form region description \cite{carreira2012semantic}, on top of manually designed local features such as SIFT. The methodology we propose makes it possible to deal with the difficulties of learning the underlying features even in the presence such a complex intermediate representation. This part is also related to kernel learning approaches over the manifold of positive-definite matrices  \cite{jayasumana2013kernel}. However, we introduce different mathematical techniques related to matrix backpropagation, which has the advantages of scalability and fitting together perfectly with existing deep network layers.

Among the first methods integrating structured models with CNNs is the work of \cite{bottou1997global} who showed that HMMs can be integrated into deep networks and showed results for speech and text analysis problems. \cite{peng2009conditional} more recently demonstrated that using CRFs and deep networks can be trained end-to-end, showing strong results on digit recognition and protein secondary structure prediction. Cast as a conditional random field (CRF) semantic segmentation has almost immediately taken advantage of the deep network revolution by providing useful smoothing on top of high-performing CNN pixel classifier predictions. \cite{long2014fully} showed that the fully connected components, usually discarded by previous methods, can also be made convolutional, \ie the original resolution lost during pooling operations can be recovered by means a trained deconvolution layer. \cite{chen2014semantic} obtained state-of-the-art semantic segmentation results using an architecture similar to \cite{long2014fully} but enforcing structure using globally connected CRFs\cite{krahenbuhl2013parameter} where only the unary potentials are learnt. Simultaneous work by \cite{Schwing15fully} and \cite{ZhengJRVSDHT_iccv15} show that, since CRF mean field based approximate updates are differentiable, a fixed number of inference steps can be unrolled, the loss can be applied to them and then the gradients can be backpropagated back first through the inference to the convolutional layers of the potentials. In \cite{chen14learning} a more efficient learning method is obtained by blending inference and training in order to obtain a procedure that updates parameters as inference progresses. Unlike previous methods \cite{Lin15Efficient} learns CNN based pairwise potentials, separate from the CNN of the unary potential. Learning the model requires piece-wise training and minimizes an upper-bound on the CRF energy that decouples the potentials.  

Our matrix backpropagation methodology generally applies to models that can be expressed as composed structured non-linear matrix functions. As such, it can be applied to these deep models with a CRF top structure as well where \eg belief propagation in models with Gaussian potentials can be expressed as a solution to a linear system\cite{tappen07learning}. While CRF-based methods designed on top of deep nets traditionally focus on iterative inference and learning where in order to construct the derivatives of the final layer, one must combine the derivatives of each inference iterations, our methodology can be expressed in terms of invariants on the converged solution of linear systems -- therefore it does not require iterative derivative calculations during inference.

Our second model used to illustrate the matrix backpropagation methodology, normalized cuts, has received less attention from the deep network community as evidenced by the fact that leading methods are still handcrafted. Spectral formulations like \emph{normalized cuts}(NCuts) \cite{shi00ncuts} have obtained state-of-the-art results when used with strong pixel-level classifiers on top of hand-designed features\cite{amfm_pami2011}. A different approach is taken in \cite{cour2005spectral} who show that MRF inference can be relaxed to a spectral problem. Turaga \emph{et al} \cite{briggman2009maximin} were the first to demonstrate the learning of an image segmentation model end-to-end using CNN features, while optimizing a standard segmentation criterion. Learning and inference of NCuts was placed on firmer ground by Bach and Jordan \cite{bach06} who introduced a (shallow) learning formulation which we build upon in this work with several important differences. First, it uses matrix derivatives, but makes appeal directly to the eigen-decompostion to derive them instead of projectors as we do. This allows them to truncate the spectrum and to consider only the eigenspace corresponding to the largest eigenvalues at the cost of (potentially) making the criterion non-differentiable. We instead consider the entire eigenspace and rely on projectors (thus on the eigen-decomposition only indirectly) and aim to learn the dimensionality in the process. More importantly however, instead of learning parameters on top of fixed features as in \cite{bach06}, we directly learn the affinity matrix by adapting the underlying feature representation, modeled as a deep network. The resulting method, combining strong pixel-level classifiers and a global (spectral) representation, can more naturally adapt pixel-level or semi-local predictions for object detection and semantic segmentation, as these operations require not only structured, global computations, but also, for consistency, propagation of the information in the image. Careful application of our methodology keeps the entire architecture trainable end-to-end. From another direction, in an effort to generalize convolutions to general non-Euclidean and non-equally spaced grids the work of \cite{bruna13} realizes the necessity of spectral layers for learning the graph structure but since the computational issues brought on in the process are not the main focus, they do not handle them directly. In \cite{henaff15} such aspects are partially addressed but the authors focus on learning parameters applied to the eigenvalues instead of learning the eigenvectors and eigenvalues as we do. In this context our focus is on the underlying theory of backpropagation when handling structured objects like matrices, allowing one to derive those and many other similar, but also more complex derivatives. 

Symbolic matrix partial derivatives, one of the basis of our work, were first systematically studied in the seminal paper \cite{dwyer1948symbolic}\footnote{With few additions from disparate sources, these are the matrix derivatives results catalogued in the collection of matrix identities called ``The Matrix Cookbook''\cite{matrixcookbook}.}, although not for complex non-linear layers like SVD or eigen-decomposition. Since then it has received interest mainly in the context of studying estimators in statistics and econometrics \cite{magnus1999matrix}. Recently, the field of \emph{automatic differentiation} has also shown interest in this theory when considering matrix functions\cite{giles2008collected}. This very powerful machinery has however appeared only scarcely in computer vision and machine learning. Some instances, although not treating the general case, and focusing on the subset of the elements (variables) of interest, appeared in the context of camera calibration\cite{papadopoulo2000estimating}, for learning parameters in a normalized cuts model\cite{bach06}, learning the parameters of Gaussian CRFs for denoising \cite{tappen07learning} and learning deep canonical correlation models \cite{andrew2013deep}. The recent surge of interest in deep networks has exposed limitations of current compositional (layered) architectures based on local operations, which in turn pushes the research in the direction of structured models requiring matrix based representations. Recently \cite{lin15bilinearcnn} multiplied the outputs of two networks as matrices, in order to obtain improved fine-grained recognition models, although the matrix derivatives in those case are straightforward. To our knowledge, we are the first to bring this methodology, in its full generality, to the fore in the context of composed global non-linear matrix functions and deep networks, and to show promising results for two different computer vision and machine learning models.

Our methodological contributions are as follows: {\it (a)} the formulation of matrix back-propagation as a generalized chain rule mechanism for computing derivatives of composed matrix functions with respect to matrix inputs (rather than scalar inputs, as in standard back-propagation), by relying on the theory of adjoint matrix variations; {\it (b)} the introduction of spectral and non-linear (global structured) layers like SVD and eigen-decomposition which allow the calculation of derivatives with respect to all the quantities of interest, in particular all the singular values and singular vectors or eigen-values and eigen-vectors, {\it (c)} the formulation of non-linear matrix function layers that take SVD or eigen-decompositions as inputs, with instantiations for second-order pooling models, {\it (d)} recipes for computing derivatives of matrix projector operations, instantiated for normalized-cuts models. {\it (e)} The novel methodology {\it (a)-(d)} applies broadly and is illustrated for end-to-end visual learning in deep networks with very competitive results.

\paragraph{Paper organization:} In the next section \S\ref{sec:deep_processing_networks} we briefly present the main elements of the current deep network models. In \S\ref{sec:matrix_backpropagation} we outline the challenges and a computational recipe to handle matrix layers. \S\ref{sec:spectral_non-linear_layers} presents a ``matrix function'' layer using either SVD or an EIG decomposition and instantiated these to build deep second-order pooling models. In \S\ref{sec:scl}, we introduce an in-depth treatment to learn deep normalized cuts models. The experiments are presented in \S\ref{sec:exps}.

\section{Deep Processing Networks}\label{sec:deep_processing_networks}

Let $\mathcal{D}=\{(\bd^{(i)},\by^{(i)})\}_{i=1\ldots N}$ be a set of data points (e.g.~images) and their corresponding desired targets (e.g.~class labels) drawn from a distribution $p(\bd,\by)$. Let $L:\reals^{d}\rightarrow\reals$ be a \emph{loss function} \ie~a penalty of mismatch between the \emph{model prediction function} $f:\reals^D\rightarrow\reals^d$ with parameters $W$ for the input $\bd$ \ie~$f(\bd^{(i)},W)$ and the desired output $\by^{(i)}$. The foundation of many learning approaches, including the ones considered here, is the principle of \emph{empirical risk minimization}, which states that under mild conditions, due to concentration of measure, the \emph{empirical risk} $\hat{R}(W)=\frac{1}{N}\sum_{i=1}^N L(f(\bd^{(i)},W),\by^{(i)})$ converges to the true risk $R(W)=\int L(f(\bd,W),\by)p(\bd,\by)$. This implies that it suffices to minimize the empirical risk to learn a function that will do well in general \ie
\begin{equation}
\argmin_W \frac{1}{N}\sum_{i=1}^N L(f(\bd^{(i)},W),\by^{(i)}) \label{eqn:learning}
\end{equation}
If $L$ and $f$ are both continuous (though not necessarily with continuous derivatives) one can use (sub-)gradient descent on \eqref{eqn:learning} for learning the parameters. This supports a general and effective framework for learning provided that a (sub-) gradient exists.

Deep networks, as a model, consider a class of functions $f$, which can be written as a series of successive function compositions $f = f^{(K)}\circ f^{(K-1)}\circ\ldots\circ f^{(1)}$ with parameter \FIX{tuple} $W=(\bw_K,\bw_{K-1},\ldots,\bw_1)$, where $f^{(l)}$ are called layers, $\bw_l$ are the parameters of layer $l$ and $K$ is the number of layers. \FIX{Denote by $L^{(l)}=L\circ f^{(K)} \circ\ldots\circ f^{(l)}$ the \emph{loss as a function of the layer $\mathbf{x}_{l-1}$}.} This \FIX{notation is convenient} because it 
conceptually separates the network architecture from the layer design. 

Since the computation of the gradient is the only requirement for learning,  \FIX{an important step is the effective use} of the principle of \emph{backpropagation} (backprop). Backprop, as described in the literature, is an algorithm to efficiently compute the gradient of the loss with respect to the parameters, in the case of layers where the outputs can be expressed as vectors of scalars, which in the most general form, can individually be expressed as non-linear functions of the input. The algorithm recursively computes gradients with respect to both the \FIX{inputs to the layers and their parameters (fig. \ref{fig:notation})} by making use of the \emph{chain rule}. For a data tuple $(\bd,\by)$ and a layer $l$ this is computing 
\FIX{\begin{equation}\dfrac{\partial L^{(l)}(\bx_{l-1},\by)}{\partial \bw_{l}}=\dfrac{\partial L^{(l+1)}(\bx_{l},\by)}{\partial \bx_{l}}\dfrac{\partial f^{(l)}(\bx_{l-1})}{\partial \bw_{l}}\label{eqn:param_deriv_new}\end{equation}
\begin{equation}\dfrac{\partial L^{(l)}(\bx_{l-1},\by)}{\partial \bx_{l-1}}=\dfrac{\partial L^{(l+1)}(\bx_{l},\by)}{\partial \bx_{l}}\dfrac{\partial f^{(l)}(\bx_{l-1})}{\partial \bx_{l-1}}\label{eqn:data_deriv_new}\end{equation} }
where \FIX{$\bx_l=f^{(l)}(\bx_{l-1})$} and $\bx_0=\bd$ (data). The first expression is the gradient we seek (required for updating $\bw_{l}$) whereas the second one is necessary for calculating the gradients in the layers below and \FIX{updating} their parameters.

\section{Structured Layers}\label{sec:sl}
The existing literature concentrates on layers of the form $f^{(l)}=(f^{(l)}_{1}(\bx_{l-1}),\ldots, f^{(l)}_{d_{l+1}}(\bx_{l-1}))$, where $f^{(l)}_{j}:\reals^{d_l}\rightarrow\reals$, thus $f^{(l)}:\reals^{d_l}\rightarrow\reals^{d_{l+1}}$. This simplifies \FIX{processing} significantly because in order to compute \FIX{$\dfrac{\partial L^{(l)}(\bx_{l-1},\by)}{\partial \bx_{l-1}}$} there is a well defined notion of partial derivative with respect to the layer $\dfrac{\partial f^{(l)}(\bx_{l-1})}{\partial \bx_{l-1}}$ as well as a simple expression for the chain rule, as given in \eqref{eqn:param_deriv_new} and \eqref{eqn:data_deriv_new}. However this formulation processes spatial coordinates independently and 
does not immediately generalize to more complex mathematical objects.

Consider a matrix view of the (3-dimensional tensor) layer, $X=\bx_{l-1}$, where $X_{ij}\in\reals$, with $i$ being the spatial coordinate\footnote{For simplicity and without loss of generality we reshape (linearize) the tensor's spatial indices to one dimension with $m_l$ coordinates.} and $j$ the index of the input feature. Then we can define a non-linearity on the entire $X\in \reals^{m_l\times d_l}$, as a matrix, instead of each (group) of spatial coordinate separately. As the matrix derivative with respect to a vector \FIX{(set aside to a matrix)} is no longer well-defined, \FIX{a matrix generalization of backpropation is necessary.} 

For clarity, one has to draw a distinction between the data structures used to represent the layers and the mathematical and computational operations performed. For example a convolutional neural network layer can be viewed, under the current implementations, as a tensor where two dimensions correspond to spatial coordinates and one dimension corresponds to features. However, all mathematical calculations at the level of layers (including forward processing or derivative calculations) are \emph{not expressed on tensors}. Instead these are performed on vectors and their scalar outputs are used to selectively index and fill the tensor data structures. In contrast, a genuine matrix calculus would not just rely on matrices as data structures, but use them as first class objects. This would require a coherent formulation where non-linear structured operations like forward propagation or derivative calculations are directly expressed using matrices. The distinction is not stylistic, as complex matrix operations for \eg SVD or eigen-decomposition and their derivatives simply cannot be implemented as index-filling.

\subsection{Computer Vision Models}
To motivate the use of structured layers we will consider the following two models from computer vision:
\begin{enumerate}
\item \emph{Second-Order Pooling} is one of the competitive hand-designed  feature descriptors for regions \cite{carreira2012semantic} used in the top-performing method of the PASCAL VOC semantic segmentation, comp.~5 track \cite{ics_nips11,li2013composite}. It represents global high-order statistics of local descriptors inside each region by computing a covariance matrix $F^\top F$, where $F\in\reals^{m\times d}$ is the matrix of image features present in the region at the $m$ spatial locations with $d$ feature dimensions, then applying a tangent space mapping \cite{Arsigny07geometric} using the matrix logarithm, which can be computed using SVD. Instead of pooling over hand-designed local descriptors, such as SIFT \cite{lowe1999object}, one could learn a deep feature extractor (\eg ConvNet) end-to-end, with an upper second-order pooling structured layer of the form 
\begin{equation}C=\log(F^\top F+\epsilon I) \label{eqn:deepo2p_svd}\end{equation} 
where $\epsilon I$ is a regularizer preventing $\log$ singularities around 0 when the covariance matrix is not full rank.
\item \emph{Normalized Cuts} is an influential global image segmentation method based on pairwise similarities \cite{shi00ncuts}. It constructs a matrix of local interactions $W=FF^\top$, where $F\in\reals^{m\times d}$ is a similar feature matrix with $m$ spatial locations and $d$ dimensions in the descriptor, then solves a generalized eigenvalue problem to determine a global image partitioning. Instead of manually designed affinities, one could, given a ground truth target segmentation, learn end-to-end the deep features that produce good normalized cuts.
\end{enumerate}

\subsection{Matrix Backpropagation}\label{sec:matrix_backpropagation}
\begin{figure}[]
\centering
\includegraphics[width=.5\columnwidth]{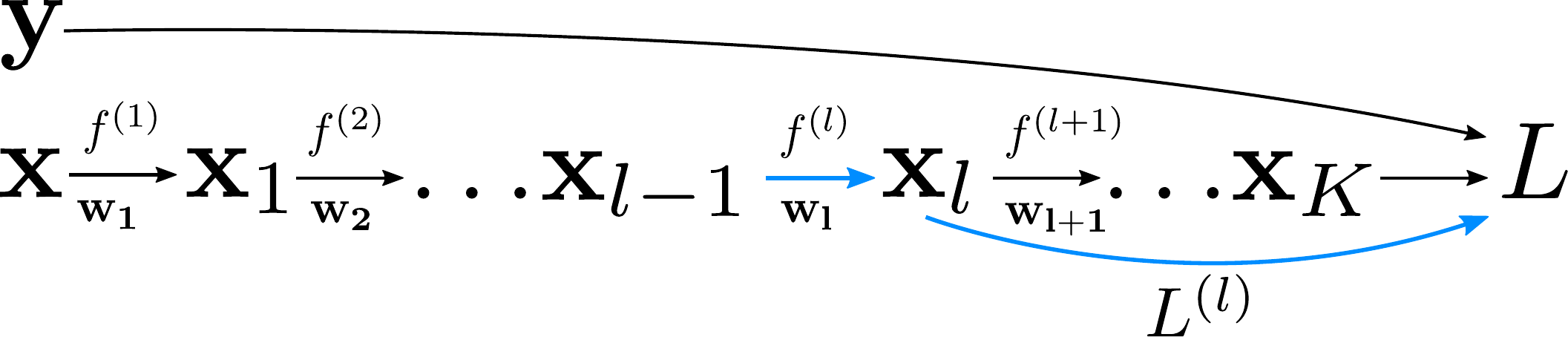}
\caption{\FIX{Deep architecture where data $\bf x$ and targets $\bf y$ are fed to a loss function $L$, via successively composed functions $f^{(l)}$ with parameters $\bw_l$. \emph{Backpropagation} (blue arrows) recursively expresses the partial derivative of the loss $L$ w.r.t. the current layer parameters based on the partial derivatives of the next layer, \cf eq. \ref{eqn:param_deriv_new}.}}
\label{fig:notation}
\end{figure}
We call \emph{matrix backpropagation} (MBP) the use of matrix calculus\cite{dwyer1948symbolic,magnus1999matrix,giles2008collected} to map between the partial derivatives $\dfrac{\partial L^{(l+1)}}{\partial \bx_{l}}$ and $\dfrac{\partial L^{(l)}}{\partial \bx_{l-1}}$ at two consecutive structured layers. Note that since for all layers $l$ the function $L^{(l)}$ maps to real numbers, by construction, both derivatives are well defined. In this section we simplify notation writing $L=L^{(l+1)}$, $X,Y$ are the matrix versions of $\bx_l$ and $\bx_{l-1}$ respectively, $f=f^{(l)}$ thus $f(X) = Y$.

The basis for the derivation is best understood starting from the Taylor expansion of the matrix functions\cite{magnus1999matrix} at the two layers
\begin{align}
 L\circ f(X+dX) -L\circ f(X)&= \frac{\partial L\circ f}{\partial X}:dX + \operatorname{O}(\lVert dX\rVert^2)\label{taylorX}\\
 L(Y+dY) - L(Y) &=  \frac{\partial L}{\partial Y}:dY + \operatorname{O}(\lVert dY\rVert^2)\label{taylorY}
\end{align}
where we introduced the notation $A:B = \Tr(A^\top B)=\vect(A)^\top\vect(B)$ for convenience. Thus $A:B$ is an inner product in the Euclidean $\vect$'d matrix space. 

Our strategy of derivation, outlined below, involves two important concepts. A \emph{variation} corresponds to the forward sensitivity and allows the easy manipulation of the first and higher order terms of a Taylor expansion. \textit{E.g.} for a matrix function $g$ we write $dg = dg(X;dX) = g(X+dX)-g(X)=A(X):dX+O(\Vert dX\Vert^2)$, with $A(X)$ a matrix of the same size as $X$ and depending on $X$ but not on $dX$. The \emph{(partial) derivative} is by definition the linear `coefficient' of a Taylor expansion \ie the coefficient of $dX$ ergo $\frac{\partial g}{\partial X}=A(X)$. The variation and the partial derivative are very different objects: $dg$ is always defined if $g$ is defined, it can take matrix inputs, and can map to a space of matrices. In contrast, the partial derivative also makes sense when $g$ has matrix inputs, but is only defined when $g$ has scalar co-domain (image)\footnote{See \cite{magnus1999matrix} for an in-depth treatment including questions about existence and uniqueness of the partial derivatives of matrix functions.}. The variation is used for the convenience of the derivation and needs not be implemented in practice. What we are ultimately after, for the purpose of matrix backpropagation, is the partial derivative. 

The important element to understand is that when
\begin{equation}
dY = df(X;dX)\label{eqn:differential_eq}
\end{equation}
the expressions \eqref{taylorX} and \eqref{taylorY} should be equal, since they both represent the variation of $L$ for a given perturbation $dX$ of the variable $X$. The first order terms of the Taylor expansions should also match, which gives us the chain rule
\begin{equation}
\frac{\partial L\circ f}{\partial X}:dX=\frac{\partial L}{\partial Y}:dY\label{eqn:chain_rule_eq}
\end{equation}
The aim is to use this identity to express the partial derivative of the left hand side as a function of the partial derivative in the right hand side. The general recipe for our derivation follows two steps\footnote{The appendices contain some of the basic identities used in each of these steps.}:
\begin{enumerate}
\item Derive the functional $\mathcal{L}$ describing the variations of the upper layer variables with respect to the variations of the lower layer variables 
\begin{equation}\label{eq:step1}
dY=\mathcal{L}(dX) \triangleq df(X;dX)
\end{equation}
The derivation of the variation involves not only the forward mapping of the layer, $f^{(l)}$, but also the invariants associated to its variables. If $X$ satisfies certain invariants, these need to be preserved to first (leading) order when computing $dX$. Invariants such as diagonality, symmetry, or orthogonality need to be explicitly enforced in our methodology, by means of additional equations (constraints) beyond \eqref{eq:step1}.

\item Given $dY$ produced in {\it step 1} above, we know that \eqref{eqn:chain_rule_eq} holds. Thus we can use the properties of the matrix inner product $A:B = \Tr(A^\top B)$ to obtain the partial derivatives with respect to the lower layer variables. 
Since the ``$:$'' operator is an inner product on the space of matrices, this is equivalent to constructively producing $\mathcal{L}^*$, a non-linear adjoint operator\footnote{For arguments $a, b$, the adjoint operator $\mathcal{L}^*$ of $\mathcal{L}$ is defined as having the property that $ a:\mathcal{L}(b)=\mathcal{L}^*(a):b$.} of $\mathcal{L}$
\begin{multline}
\frac{\partial L}{\partial Y}:dY = \frac{\partial L}{\partial Y}:\mathcal{L}(dX) \triangleq \mathcal{L}^*\left(\frac{\partial L}{\partial Y}\right):dX
\Rightarrow\, \mathcal{L}^*\left(\frac{\partial L}{\partial Y}\right) = \frac{\partial L\circ f}{\partial X}\hfill\text{ (by the chain rule)}
\end{multline}

This holds for a \emph{general variation}, e.g.~for a non-symmetric $dX$ even if $X$ itself is symmetric. To remain within a subspace like the one of symmetric, diagonal or orthogonal matrices, we consider a projection of $dX$ onto the space of admissible variations and then transfer the projection onto the derivative, to obtain the \emph{projected gradient}. We use this technique repeatedly in our derivations.

\end{enumerate}

\noindent Summarizing, the objective of our calculation is to obtain $\frac{\partial L\circ f}{\partial X}$. Specifically, we will compute $\frac{\partial L}{\partial Y}$ (typically back-propagated from the layer above) and $dY=\mathcal{L}(dX)$, then process the resulting expression using matrix identities, in order to obtain an analytic expression for $\frac{\partial L}{\partial Y}:\mathcal{L}(dX)$. In turn, extracting the inner product terms $\mathcal{L}^*\left(\frac{\partial L}{\partial Y}\right):dX$ from that expression, allows us to compute $\mathcal{L}^*$.

\section{Spectral and Non-Linear Layers}\label{sec:spectral_non-linear_layers}

\FIX{When global matrix operations are used in deep networks, they compound with other processing layers performed along the way. Such steps are architecture specific, although calculations like spectral decomposition are widespread, and central, in many vision and machine learning models. SVD possesses a powerful structure that allows one to express complex transformations like matrix functions and algorithms in a numerically stable form. In the sequel we show how the widely useful \emph{singular value decomposition} (SVD) and the \emph{symmetric eigenvalue problem} (EIG) can be leveraged towards constructing layers that perform global calculations in deep networks.} 

\subsection{Spectral Layers}

The SVD layer receives a matrix $X$ as input and produces a tuple of 3 matrices $U$,$\Sigma$ and $V$. Under the notation above, this means $Y=f(X)=(U,\Sigma,V)$. The matrices satisfy the regular invariants of the SVD decomposition \ie $X=U\Sigma V^\top$, $U^\top U = I$, $V^\top V=I$ and $\Sigma$ is diagonal which will be taken into account in the derivation. The following proposition gives the variations of the outputs \ie $\mathcal{L}(dX)=dY=(dU, d\Sigma, dV)$ and the partial derivative with respect to the layer $\dfrac{\partial L\circ f}{\partial X}$ as a function of the partial derivatives of the outputs $\frac{\partial L}{\partial Y}$, \ie  $\dfrac{\partial L}{\partial U}$,$\dfrac{\partial L}{\partial \Sigma}$ and $\dfrac{\partial L}{\partial V}$. Note that these correspond, respectively, to the first and second step of the methodology presented in \S\ref{sec:matrix_backpropagation}. In the sequel, we denote $A_{sym}=\frac{1}{2}(A^\top+A)$ and $A_{diag}$ be $A$ with all off-diagonal elements set to 0.

\begin{prop}[SVD Variations]\label{prop:svd}Let $X=U\Sigma V^\top$ with $X\in\reals^{m,n}$ and $m\geq n$, such that $U^\top U=I$, $V^\top V=I$ and $\Sigma$ possessing diagonal structure. Then 
\begin{equation}d\Sigma = (U^\top dX V)_{diag}\label{eqn:svd_dS}\end{equation} and 
\begin{equation}dV=2V\left(K^\top\circ(\Sigma^\top U^\top dX V)_{sym}\right) \label{eqn:svd_dV}\end{equation} with
\begin{equation}
 K_{ij} = \begin{cases} \dfrac{1}{\sigma_i^2-\sigma_j^2}, &i\neq j\\
 0, &i=j
 \end{cases}
\end{equation}
Let $\Sigma_n\in\reals^{n\times n}$ be the top $n$ rows of $\Sigma$ and consider the block decomposition $dU=\bigl(dU_1\,|\, dU_2 \bigr)$ with $dU_1\in\reals^{m\times n}$ and $dU_2\in\reals^{m\times m-n}$ and similarly $\dfrac{\partial L}{\partial U} = \left(\left(\dfrac{\partial L}{\partial U}\right)_1\,\middle|\, \left(\dfrac{\partial L}{\partial U}\right)_2\right)$, where $\left(\dfrac{\partial L}{\partial U}\right)_1\in\reals^{m\times n}$ and $\left(\dfrac{\partial L}{\partial U}\right)_2\in\reals^{m\times m-n}$. Then 
\begin{equation}
dU = \bigl(C\Sigma_n^{-1}\,|\, -U_1\Sigma_n^{-1}C^\top U_2 \bigr)
\end{equation}
with
\begin{equation}
  C= dX V - Ud\Sigma - U\Sigma dV^\top V 
\end{equation}
Consequently the partial derivatives are
\begin{equation}
\frac{\partial L\circ f}{\partial X} = DV^\top+U\left(\frac{\partial L}{\partial \Sigma}-U^\top D\right)_{diag}V^\top+2U\Sigma\left(K^\top\circ\left(V^\top\left(\frac{\partial L}{\partial V}-VD^\top U\Sigma\right)\right)\right)_{sym}V^\top\label{eqn:dLdX}
\end{equation}
where 
\begin{equation}
D=\left(\dfrac{\partial L}{\partial U}\right)_1\Sigma_n^{-1}-U_2\left(\dfrac{\partial L}{\partial U}\right)_2^\top U_1\Sigma_n^{-1}
\end{equation}
\end{prop}
\begin{proof}
Let $X = U \Sigma V^\top$ by way of SVD, with $X\in\mathbb{R}^{m\times n}$ and $m\geq n$, $\Sigma\in \mathbb{R}^{m\times n}$ diagonal and $U\in\mathbb{R}^{m\times m}$, $V\in\mathbb{R}^{n\times n}$ orthogonal. For a given variation $dX$ of $X$, we want to calculate the variations $dU$,$d\Sigma$ and $dV$. The variation $d\Sigma$ is diagonal, like $\Sigma$, whereas $dU$ and $dV$ satisfy (by orthogonality) the constraints $U^\top dU + dU^\top U = 0$ and $V^\top dV + dV^\top V = 0$ respectively. Taking the first variation of the SVD decomposition, we have
\begin{equation}
dX = dU\Sigma V^\top + Ud\Sigma V^\top + U\Sigma dV^\top\label{eqn:dX_SVD}
\end{equation}
then, by using the orthogonality of $U$ and $V$, it follows that
\begin{align*}
 \Rightarrow &U^\top dXV = U^\top dU\Sigma + d\Sigma + \Sigma dV^\top V \Rightarrow \\ 
 \Rightarrow &R = A\Sigma + d\Sigma + \Sigma B
\end{align*}

with $R=U^\top dXV$ and $A=U^\top dU$, $B=dV^\top V$ both antisymmetric. Since $d\Sigma$ is diagonal whereas $A\Sigma$, $\Sigma B$ have both zero diagonal, we conclude that 
\begin{equation}{d\Sigma = (U^\top dX V)_{diag}}\label{svd_dSigma}\end{equation} 

The off-diagonal part then satisfies
\begin{multline}
 A\Sigma + \Sigma B = R-R_{diag} \Rightarrow \Sigma^\top A \Sigma + \Sigma^\top\Sigma B = \Sigma^\top(R-R_{diag}) = \Sigma^\top \bar{R}\\
 \Rightarrow \begin{cases}
  \sigma_i a_{ij} \sigma_j + \sigma_i^2 b_{ij} = \sigma_i {\bar R}_{ij}\\
  -\sigma_j a_{ij} \sigma_i - \sigma_j^2 b_{ij} = \sigma_j {\bar R}_{ji} &\text{($A$,$B$ antisym.)}
 \end{cases}\\ \Rightarrow (\sigma_i^2 - \sigma_j^2)b_{ij} = \sigma_i \bar{R}_{ij} + \bar{R}_{ji}\sigma_j \Rightarrow b_{ij} = \begin{cases} (\sigma_i^2 - \sigma_j^2)^{-1}\left(\sigma_i \bar{R}_{ij} + \bar{R}_{ji}\sigma_j\right)\, , & i\neq j\\
 0\, , & i=j\end{cases}
\end{multline}
where $\sigma_i = \Sigma_{ii}$ and $\bar{R}=R-R_{diag}$. We can write this as $B=K\circ(\Sigma^\top\bar{R} + \bar{R}^\top\Sigma)=K\circ(\Sigma^\top R + R^\top\Sigma)$, where 
\begin{equation}
 K_{ij} = \begin{cases} \dfrac{1}{\sigma_i^2-\sigma_j^2}, &i\neq j\\
 0, &i=j
 \end{cases}\label{svd_K}
\end{equation}
Finally, 
\begin{equation}dV=VB^\top \Rightarrow dV=2V\left(K^\top\circ(\Sigma^\top U^\top dX V)_{sym}\right)\label{svd_dV}\end{equation}
Note that this satisfies the condition $V^\top dV + dV^\top V=0$ by construction, and so preserves the orthogonality of $V$ to leading order.


Using the $d\Sigma$ and $dV$ we have obtained, one can obtain $dU$ from the variations of $dX$ in \eqref{eqn:dX_SVD}:
\begin{equation*}
  dX = dU\Sigma V^\top + Ud\Sigma V^\top + U\Sigma dV^\top \Rightarrow dU\Sigma = dX V - Ud\Sigma - U\Sigma dV^\top V =: C
\end{equation*}
This equation admits any solution of the block form $dU=\bigl(dU_1\, dU_2 \bigr)$, where $dU_1:=C\Sigma_n^{-1}\in\reals^{m\times n}$ ($\Sigma_n$ being the top $n$ rows of $\Sigma$) and $dU_2\in\reals^{m\times m-n}$ arbitrary as introduced in the proposition. To determine $dU_2$ uniquely we turn to the orthogonality condition
\begin{equation*}
 dU^\top U + U^\top dU=0 \Rightarrow \begin{pmatrix} dU_1^\top U_1 + U_1^\top dU_1 & dU_1^\top U_2 + U_1^\top dU_2  \\ dU_2^\top U_1 + U_2^\top dU_1 & dU_2^\top U_2 + U_2^\top dU_2  \end{pmatrix} = 0
\end{equation*}
The block $dU_1$ already satisfies the top left equation, so we look at the top right (which is equivalent to bottom left). Noting that $U_1^\top U_1 = I$ by the orthogonality of $U$, we can verify that $dU_2 = -U_1 dU_1^\top U_2$. Since this also satisfies the remaining equation, orthogonality is satisfied. Summarizing
\begin{equation}
 dU = \bigl(C\Sigma_n^{-1}\,|\, -U_1\Sigma_n^{-1}C^\top U_2 \bigr), \quad C= dX V - Ud\Sigma - U\Sigma dV^\top V 
\end{equation}

We proceed further with the second part of the matrix backpropagation to compute $\dfrac{\partial L\circ f}{\partial X}$. Note that the chain rule in this case is $\dfrac{\partial L\circ f}{\partial X}:dX=\dfrac{\partial L}{\partial U}:dU+\dfrac{\partial L}{\partial \Sigma}:d\Sigma+\dfrac{\partial L}{\partial V}:dV$. We can replace $dU$, $dV$ and $d\Sigma$ with their expressions w.r.t.~$dX$ to obtain the partial derivatives. 

Before computing the full derivatives we simplify slightly the expression corresponding to $dU$
\begin{align}
\dfrac{\partial L}{\partial U}:dU&=\left(\dfrac{\partial L}{\partial U}\right)_1:C\Sigma_n^{-1} + \left(\dfrac{\partial L}{\partial U}\right)_2:-U_1\Sigma_n^{-1}C^\top U_2\\
&=\left(\dfrac{\partial L}{\partial U}\right)_1\Sigma_n^{-1}:C - \Sigma^{-1}_n U_1^\top\left(\dfrac{\partial L}{\partial U}\right)_2 U_2^\top:C^\top\\
&=\left(\dfrac{\partial L}{\partial U}\right)_1\Sigma_n^{-1}:C - U_2\left(\dfrac{\partial L}{\partial U}\right)_2^\top U_1\Sigma_n^{-1}:C\\
&=\left(\left(\dfrac{\partial L}{\partial U}\right)_1\Sigma_n^{-1}-U_2\left(\dfrac{\partial L}{\partial U}\right)_2^\top U_1\Sigma_n^{-1}\right):(dX V - Ud\Sigma - U\Sigma dV^\top V )\\
&=D:dX V - D:U d\Sigma - D:U\Sigma dV^\top V\\
&=DV^\top:dX-U^\top D:d\Sigma - \Sigma U^\top DV^\top:dV^\top\\
&=DV^\top:dX-U^\top D:d\Sigma - VD^\top U\Sigma :dV
\end{align}

Now we can plug this result in the full derivative
\begin{align*}
 \frac{\partial L}{\partial U}:dU  + \frac{\partial L}{\partial\Sigma}:d \Sigma+ \frac{\partial L}{\partial V}:dV=
                & (DV^\top:dX-U^\top D:d\Sigma - VD^\top U\Sigma:dV) +\frac{\partial L}{\partial \Sigma}:(U^\top dX V)_{diag}+\\
                &+\frac{\partial L}{\partial V}:\left\{2V\left(K^\top\circ(\Sigma^\top U^\top dX V)_{sym}\right)\right\}\\
  				=& DV^\top:dX +\left(\frac{\partial L}{\partial \Sigma}-U^\top D\right):(U^\top dX V)_{diag}+\\
  				&+\left(\frac{\partial L}{\partial V}-VD^\top U\Sigma\right):\left\{2V\left(K^\top\circ(\Sigma^\top U^\top dX V)_{sym}\right)\right\}\\
  				=& DV^\top:dX +\left(\frac{\partial L}{\partial \Sigma}-U^\top D\right)_{diag}:(U^\top dX V)+\\
  				&+2V^\top\left(\frac{\partial L}{\partial V}-VD^\top U\Sigma\right):\left(K^\top\circ(\Sigma^\top U^\top dX V)_{sym}\right)~~~~~~~~~~~~~~ \text{by \eqref{col:prod}, \eqref{col:diag}}\\
 				=&DV^\top:dX +U\left(\frac{\partial L}{\partial \Sigma}-U^\top D\right)_{diag}V^\top:dX+\\
  				&+2\left(K^\top\circ\left(V^\top\left(\frac{\partial L}{\partial V}-VD^\top U\Sigma\right)\right)\right)_{sym}:\Sigma^\top U^\top dX V~~~~~~~~ \text{by \eqref{col:sym},\eqref{col:circ}}\\
  				=&DV^\top:dX +U\left(\frac{\partial L}{\partial \Sigma}-U^\top D\right)_{diag}V^\top:dX+\\
  				&+2U\Sigma\left(K^\top\circ\left(V^\top\left(\frac{\partial L}{\partial V}-VD^\top U\Sigma\right)\right)\right)_{sym}V^\top:dX ~~~~~~~~~~~~~~~~~ \text{by \eqref{col:prod}}\\
\end{align*}
and so, since the last expression is equal to $\dfrac{\partial L\circ f}{\partial X}:dX$ by the chain rule,
\begin{equation}\label{dLdX}
 {\frac{\partial L\circ f}{\partial X} = DV^\top+U\left(\frac{\partial L}{\partial \Sigma}-U^\top D\right)_{diag}V^\top+2U\Sigma\left(K^\top\circ\left(V^\top\left(\frac{\partial L}{\partial V}-VD^\top U\Sigma\right)\right)\right)_{sym}V^\top}
\end{equation}
\end{proof}

The EIG is a layer that receives a matrix $X$ as input and produces a pair of matrices $U$ and $\Sigma$. Given our notation, this means $Y=f(X)=(U,\Sigma)$. The matrices satisfy the regular invariants of the eigen-decomposition \ie $X=U\Sigma U^\top$, $U^\top U = I$ and $\Sigma$ is a diagonal matrix. The following proposition identifies the variations of the outputs \ie $\mathcal{L}(dX)=dY=(dU,d\Sigma)$ and the partial derivative with respect to this layer $\dfrac{\partial L\circ f}{\partial X}$ as a function of the partial derivatives of the outputs $\dfrac{\partial L}{\partial Y}$ \ie $\dfrac{\partial L}{\partial U}$ and $\dfrac{\partial L}{\partial \Sigma}$. 

\begin{prop}[EIG Variations]Let $X=U\Sigma U^\top$ with $X\in\reals^{m,m}$, such that $U^\top U=I$ and $\Sigma$ possessing diagonal structure. Then \begin{equation}d\Sigma = (U^\top dX U)_{diag}\label{eqn:eig_dS}\end{equation}
and
\begin{equation} dU=U\left(K^\top\circ(U^\top dX U)\right)\end{equation} with
\begin{equation}
 \tilde{K}_{ij} = \begin{cases} \dfrac{1}{\sigma_i-\sigma_j}, &i\neq j\\
 0, &i=j
 \end{cases}
\end{equation}
The resulting partial derivatives are
\begin{equation}
\dfrac{\partial L\circ f}{\partial X}= U\left\{\left(\tilde{K}^\top \circ \left(U^\top\dfrac{\partial L}{\partial U}\right)\right) + \left(\dfrac{\partial L}{\partial\Sigma}\right)_{diag}\right\}U^\top
\end{equation}
\end{prop}
\begin{proof}

First note that \eqref{svd_dSigma} still holds and with the notation above we have in this case $m=n$, $U=V$. This implies
\begin{equation}
{d\Sigma = (U^\top dX U)_{diag}}\label{eig_dSigma}
\end{equation}

Furthermore we have $A=B^\top$ ($A$, $B$ antisymmetric) and the off-diagonal part then satisfies $A\Sigma + \Sigma A^\top = R-R_{diag} $. In a similar process with the asymmetric case, we have
\begin{multline*}
 A\Sigma+\Sigma A^\top= R-R_{diag} \Rightarrow A\Sigma-\Sigma A= \bar{R}
  \Rightarrow \begin{cases} 
  a_{ij} \sigma_j - a_{ij} \sigma_i = \bar{R}_{ij}, & i\neq j\\
  a_{ij} = 0, & i=j\\
  \end{cases}
\end{multline*}
so that $A=\tilde{K}^\top\circ\bar{R}$ with 
\begin{equation}
 \tilde{K}_{ij} = \begin{cases} \dfrac{1}{\sigma_i-\sigma_j}, &i\neq j\\
 0, &i=j
 \end{cases}\label{svd_Ktilde}
\end{equation}
From this, we get then \begin{equation}{dU=U\left(\tilde{K}^\top\circ(U^\top dX U)\right)}\label{svd_dU}\end{equation} 

Note that the chain rule in this case is $\dfrac{\partial L\circ f}{\partial X}:dX=\dfrac{\partial L}{\partial U}:dU+\dfrac{\partial L}{\partial \Sigma}:d\Sigma$, we can replace $dU$ and $d\Sigma$ with their expressions w.r.t.~$dX$ to obtain the partial derivatives

\begin{align*}
 \dfrac{\partial L}{\partial U}:dU + \dfrac{\partial L}{\partial \Sigma}:d\Sigma& = \dfrac{\partial L}{\partial U}:\left\{U\left(\tilde{K}^\top\circ\left(U^\top dX U\right)\right)\right\} + \frac{\partial L}{\partial \Sigma}:\left(U^\top dX U\right)_{diag}\\
  &= U\left(\tilde{K}^\top \circ \left(U^\top\dfrac{\partial L}{\partial U}\right)\right)U^\top: dX+ U\left(\dfrac{\partial L}{\partial \Sigma}\right)_{diag}U^\top:dX\\  
  \end{align*}
and so
\begin{equation}
\hspace{-10pt}{\dfrac{\partial L\circ f}{\partial X}= U\left\{\left(\tilde{K}^\top \circ \left(U^\top\dfrac{\partial L}{\partial U}\right)\right) + \left(\dfrac{\partial L}{\partial \Sigma}\right)_{diag}\right\}U^\top}\label{svd_dLdZ}
\end{equation}	
Note that \eqref{svd_dSigma}, \eqref{svd_dU}, \eqref{svd_Ktilde} and \eqref{svd_dLdZ} represent the desired quantities of the proposition.
\end{proof}

\subsection{Non-Linear Layers}
Using the SVD and EIG layers presented above we are now ready to produce layers like O$_2$P that involve matrix functions $\fun$, \eg~$\fun=\log$, but that are learned end-to-end. To see how, consider the SVD of some deep feature matrix $F=U\Sigma V^\top$ and notice that $\fun(F^\top F + \epsilon I) = \fun(V \Sigma^\top U^\top U \Sigma V^\top+\epsilon VV^\top)= V\fun(\Sigma^\top \Sigma+\epsilon I)V^\top$, where the last equality is obtained from the definition of matrix functions given that Schur decomposition and the eigen-decomposition coincide for real symmetric matrices\cite{Golub96}. Thus to implement the matrix function, we can create a new layer that receives the outputs of the SVD layer and produces $V\fun(\Sigma^\top \Sigma+\epsilon I)V^\top$, where $\fun$ is now applied element-wise to the diagonal elements of $\Sigma^\top \Sigma+\epsilon I$ thus is much easier to handle.

An SVD matrix function layer receives as input a tuple of 3 matrices $U$,$\Sigma$ and $V$ and produces the response $C=V \fun(\Sigma^\top\Sigma + \epsilon I) V^\top$, where $\fun$ is an analytic function and $\epsilon$ is a parameter that we consider fixed for simplicity. With the notation in section  \S\ref{sec:matrix_backpropagation} we have $X=(U,\Sigma,V)$ and $Y=f(X)=V \fun(\Sigma^\top\Sigma + \epsilon I) V^\top$. The following proposition gives the variations of the outputs are \ie $\mathcal{L}(dX)=dY=dC$ and the partial derivatives with respect to this layer $\dfrac{\partial L\circ f}{\partial X}$ \ie $\left(\dfrac{\partial L\circ f}{\partial V},\dfrac{\partial L\circ f}{\partial \Sigma}\right)$ as a function of the partial derivatives of the outputs $\dfrac{\partial L}{\partial C}$. Note that the layer does not depend on $U$ so its derivative $\dfrac{\partial L}{\partial U}=0$. 

\begin{prop}[SVD matrix function]
An (analytic) matrix function of a diagonalizable matrix $A=V\Sigma V^\top$ can be written as $\fun(A)=V \fun(\Sigma) V^\top$. Since $\Sigma$ is diagonal this is equivalent to applying $\fun$ element-wise to $\Sigma$'s diagonal elements. Combining this idea with the SVD decomposition $F=U\Sigma V^\top$, our matrix function can be written as $C=\fun(F^\top F+\epsilon I)=V\fun(\Sigma^\top\Sigma+\epsilon I)V^\top$.

Then the variations are
\begin{align*}
dC &= 2\left(dV~\fun(\Sigma^\top\Sigma+\epsilon I)V^\top\right)_{sym} + 2\left(V \fundev(\Sigma^\top\Sigma+\epsilon I)\Sigma^\top d\Sigma V^\top\right)_{sym}
\end{align*}
and the partial derivatives are
\begin{equation}\label{eqn:dLdV}
  \frac{\partial L\circ f}{\partial V} = 2\left(\frac{\partial L}{\partial C}\right)_{sym}V \fun(\Sigma^\top\Sigma+\epsilon I)
\end{equation}
and
\begin{equation}\label{eqn:dLdS}
  \frac{\partial L\circ f}{\partial\Sigma} = 2\Sigma\fundev(\Sigma^\top\Sigma+\epsilon I) V^\top\left(\frac{\partial L}{\partial C}\right)_{sym}V
\end{equation}
\end{prop}
\begin{proof}
Using the fact that for a positive diagonal matrix $A$ and a diagonal variation $dA$, $\fun(A+dA) = \fun(A) + \fundev(A) dA + \operatorname{O}(\Vert dA\Vert^2)$, we can write \begin{align*}
dC &= 2\left(dV\fun(\Sigma^\top\Sigma+\epsilon I)V\right)_{sym} + 2\left(V\fundev(\Sigma^\top\Sigma+\epsilon I)\Sigma^\top d\Sigma V^\top\right)_{sym}
\end{align*}

The total variation $dL$ of an expression of the form $L=\fun(C)$, $\fun:\reals^{n\times n}\rightarrow\reals^{n\times n}$, can then be written as:
\begin{align*}
 \frac{\partial L}{\partial C}:dC =& \frac{\partial L}{\partial C}:\left\{2\left(dV\fun(\Sigma^\top\Sigma+\epsilon I)V^\top\right)_{sym} + 2\left(V\fundev(\Sigma^\top\Sigma+\epsilon I)\Sigma^\top d\Sigma V^\top\right)_{sym} \right\} \\
 				=& 2\left(\frac{\partial L}{\partial C}\right)_{sym}:(dV\fun(\Sigma^\top\Sigma+\epsilon I)V^\top)  + 2\left(\frac{\partial L}{\partial C}\right)_{sym}:(V\fundev(\Sigma^\top\Sigma+\epsilon I)\Sigma^\top d\Sigma V^\top) ~~~~~~~~~~~~\text{by \eqref{col:sym}}\\
				=& 2\left\{\left(\frac{\partial L}{\partial C}\right)_{sym}V \fun(\Sigma^\top\Sigma+\epsilon I)\right\}:dV + 2\left\{\Sigma\fundev(\Sigma^\top\Sigma+\epsilon I) V^\top\left(\frac{\partial L}{\partial C}\right)_{sym}V\right\}:d\Sigma ~~~~~~~~\text{by \eqref{col:prod}}
\end{align*}

By the chain rule, we must have 
\begin{multline}\label{fun_cov_partials}
 \frac{\partial L}{\partial C}:dC = \frac{\partial L\circ f}{\partial V}:dV + \frac{\partial L\circ f}{\partial \Sigma}:d\Sigma  \Rightarrow \begin{cases}
  {\frac{\partial L\circ f}{\partial V} = 2\left(\frac{\partial L}{\partial C}\right)_{sym}V \fun(\Sigma^\top\Sigma+\epsilon I)}\\
  {\frac{\partial L\circ f}{\partial \Sigma} = 2\Sigma\fundev(\Sigma^\top\Sigma+\epsilon I) V^\top\left(\frac{\partial L}{\partial C}\right)_{sym}V}
 \end{cases}
\end{multline}
\end{proof}


Similarly the EIG matrix function layer receives as input a pair of matrices $U$ and $Q$ and produces the response $C=U \fun(Q) U^\top$. With the notation from  \S\ref{sec:matrix_backpropagation} we have $X=(U,Q)$ and $Y=f(X)=U \fun(Q) U^\top$. Note that if the inputs obey the invariants of the EIG decomposition of some real symmetric matrix $Z=UQU^\top$ \ie $U$ are the eigenvectors and $Q$ the eigenvalues, then the layer produces the result of the matrix function $C=\fun(Z)$. This holds for similar reasons as above $\fun(Z)=\fun(UQU^\top)=U\fun(Q)U^\top$, since in this case the Schur decomposition coincides with the eigen-decomposition \cite{Golub96}. The following proposition shows what the variations of the outputs are in this case \ie $\mathcal{L}(dX)=dY=dC$ and what the partial derivatives with respect to this layer are $\dfrac{\partial L\circ f}{\partial X}$ \ie $\left(\dfrac{\partial L\circ f}{\partial U},\dfrac{\partial L\circ f}{\partial Q}\right)$ as a function of the partial derivatives of the outputs $\dfrac{\partial L}{\partial C}$. 

\begin{prop}[EIG matrix function]
Let $Z = U Q U^\top$ by way of eigen-decomposition (symmetric SVD),
with $Z\in S_+(m)$ an $m\times m$ real symmetric matrix. Then $Q\in \mathbb{R}^{m\times m}$ is diagonal (the \emph{strictly positive} eigenvalues) and $U\in\mathbb{R}^{m\times m}$ is orthogonal (the eigenvectors). Denote with $C = \fun(Z)= U\fun(Q)U^\top$
Then the variations of $C$ are given by
\begin{equation}
dC = 2(dU \fun(Q)U^\top)_{sym}+U\fundev(Q)dQU^\top \label{eqn:eigfun_variations}
\end{equation}
and the partial derivatives are 
\begin{align}
\dfrac{\partial L\circ f}{\partial U} = 2\left(\dfrac{\partial L}{\partial C}\right)_{sym}U\fun(Q)\label{eqn:dLdU}\\
\dfrac{\partial L\circ f}{\partial Q} = \fundev(Q)U^\top\dfrac{\partial L}{\partial C}U\label{eqn:dLdQ}
\end{align}
\end{prop}
\begin{proof}
The variation of $C$ is 
\begin{multline*}
dC = dU \fun (Q) U^\top + U d\fun(Q) U^\top + U \fun(Q) dU^\top
\Rightarrow {dC=2(dU \fun(Q)U^\top)_{sym}+U\fundev(Q)dQU^\top}
\end{multline*}

We consider the variation of $L$,
\begin{align*}
\dfrac{\partial L}{\partial C}:dC &= \dfrac{\partial L}{\partial C}:\left\{ 2(dU \fun(Q)U^\top)_{sym}+U\fundev(Q)dQU^\top \right\}\\
				 &= \fundev(Q)U^\top\dfrac{\partial L}{\partial C}U :dQ + 2\left(\dfrac{\partial L}{\partial C}\right)_{sym}U\fun(Q):dU
\end{align*}				 
By the chain rule, we must have 
\begin{align*}
 \dfrac{\partial L}{\partial C}:dC = \dfrac{\partial L\circ f}{\partial U}:dU + \frac{\partial L\circ f}{\partial Q}:dQ  \Rightarrow \begin{cases}
  {\dfrac{\partial L\circ f}{\partial U} = 2\left(\dfrac{\partial L}{\partial C}\right)_{sym}U\fun(Q)}\\
  {\dfrac{\partial L\circ f}{\partial Q} = \fundev(Q)U^\top\dfrac{\partial L}{\partial C}U}
 \end{cases}
\end{align*}

\end{proof}

Now it is trivial to derive two versions of the O$_2$P descriptor \eqref{eqn:deepo2p_svd} by plugging in $\log$ and its derivative in the propositions above.
\begin{Corollary}[DeepO$_2$P]
Deep O$_2$P layers can be implemented and have the following backpropagation rules
\begin{enumerate}
\item DeepO$_2$P-SVD:
\begin{equation}
 \frac{\partial L\circ f}{\partial V} = 2\left(\frac{\partial L}{\partial C}\right)_{sym}V \log(\Sigma^\top\Sigma+\epsilon I)\text{~~and~~}\frac{\partial L\circ f}{\partial \Sigma} = 2\Sigma(\Sigma^\top\Sigma+\epsilon I)^{-1} V^\top\left(\frac{\partial L}{\partial C}\right)_{sym}V
\end{equation}
\item DeepO$_2$P-EIG:
\begin{equation}
{\dfrac{\partial L\circ f}{\partial U} = 2\left(\dfrac{\partial L}{\partial C}\right)_{sym}U\log(Q)}\text{~~and~~}{\dfrac{\partial L\circ f}{\partial Q} = Q^{-1}\left(U^\top\dfrac{\partial L}{\partial C}U\right)}
\end{equation}

\end{enumerate}
\end{Corollary}
\begin{proof}
If $\fun(A) = \log(A)$ then $\fundev(A)=A^{-1}$. Plugging these into \eqref{eqn:dLdV} and \eqref{eqn:dLdS} we obtain the DeepO$_2$P-SVD derivatives above. Similarly, plugging into \eqref{eqn:dLdU} and \eqref{eqn:dLdQ} gives the DeepO$_2$P-EIG derivatives.
\end{proof}

\section{Normalized Cuts Layers}\label{sec:scl}

A central computer vision and machine problem is grouping, segmentation or clustering, \ie~discovering which datapoints (or pixels) belong to one of several partitions. A successful approach to clustering is \emph{normalized cuts}. Let $m$ be the number of pixels in the image and let $V=\{1,\dots,m\}$ be the set of indices. We want to compute a partition $\mathcal{P}=\{P_1,\ldots,P_k\}$, where $k=|\mathcal{P}|$, $P_i\subset V$, $\bigcup_i P_i=V$ and $P_j\bigcap P_i=\varnothing$. This is equivalent to producing a matrix $E\in\{0,1\}^{m\times k}$ such that $E(i,j)=1$ if $i\in P_j$ and 0 otherwise. Let $F\in\reals^{m\times d}$ be a data feature matrix with descriptor of size $d$ and let $W$ be a data similarity matrix with positive entries. For simplicity we consider $W=F\Lambda F^\top$, where $\Lambda$ is a $d\times d$ parameter matrix. Note that one can also apply other global non-linearities on top of the segmentation layer, as presented in the previous section. Let $D=[W\one]$, where $[\bv]$ is the diagonal matrix with main diagonal $\bv$, \ie the diagonal elements of $D$ are the sums of the corresponding rows of $W$.
The \emph{normalized cuts criterion} is then
\begin{equation}
C(W,E)=\Tr(E^\top W E (E^\top D E)^{-1})\label{eqn:nc}
\end{equation}
Finding the matrix $E$ that minimizes $C(W,E)$ is equivalent to finding a partition that minimizes the cut energy but penalizes unbalanced solutions. 

It is easy to show\cite{bach06} that $C(W,E)=k-\Tr(Z^\top D^{-1/2}WD^{-1/2}Z)$, where $Z$ is such that: \textit{a)} $Z^\top Z = I$, and \textit{b)} $D^{1/2}Z$ is piecewise constant with respect to $E$ (\ie~it is equal to $E$ times some scaling for each column). Ignoring the second condition we obtain a relaxed problem that can be solved, due to Ky Fan theorem, by an eigen-decomposition of
\begin{equation}
M=D^{-1/2}WD^{-1/2}\label{seg:def_M}
\end{equation}
\cite{bach06} propose to learn the parameters $\Lambda$ such that $D^{1/2}Z$ is piecewise constant because then, solving the relaxed problem is equivalent to the original problem. In \cite{bach06} the input features were fixed, thus $\Lambda$ are the only parameters to permit the alignment. This is not our case, as we place a global objective on top of convolutional network inputs. We can can therefore take leverage the network parameters in order to change $F$ directly, thus training the bottom layers to produce a representation that is appropriate for normalized cuts.

To obtain a $Z$ that is piecewise constant with respect to $D^{1/2}E$ we can align the span of $M$ with that of $\Omega=D^{1/2}E E^\top D^{1/2}$. For this we can use projectors $\Pi_A$ of the corresponding space spanned by $A$, where $\Pi_A=AA^+$ is an orthogonal projector and $A^+$ is the Moore-Penrose inverse of $A$. The alignment is achieved by minimizing the Frobenius norm of the projectors associated to the the model prediction $\Pi_M$ and the desired output $\Pi_\Omega$, respectively
\begin{equation}
J_1(W,E)=\frac{1}{2}\left\Vert \Pi_M-\Pi_\Omega\right\Vert_{F}^2\label{seg:align}
\end{equation}
Notice that while our criterion $J_1$ is superficially similar to the one in \cite{bach06}, there are important differences. \cite{bach06} truncate the spectrum and consider only the eigenspace corresponding to the largest eigenvalues at the cost of (potentially) making the criterion non-differentiable. In contrast, we consider the entire eigenspace and rely on projectors (and only indirectly on eigen-decomposition) aiming to also learn the dimensionality of the space in the process.

We will obtain the partial derivatives of an objective with respect to the matrices it depends on, relying on matrix backpropagation. Since the projectors will play a very important role in a number of different places in this section we will treat them separately.

Consider a layer that takes a matrix $A$ as input and produces its corresponding orthogonal projector $\Pi_A=AA^+$. In the notation of section \ref{sec:matrix_backpropagation}, $X=A$ and $Y=f(A)=\Pi_A$. The following proposition gives the variations of the outputs \ie $\mathcal{L}(dX)=dY=d\Pi_A$ and the partial derivative with respect to the layer $\dfrac{\partial L\circ f}{\partial X}$ as a function of the partial derivatives of the outputs $\frac{\partial L}{\partial Y}$, \ie  $\dfrac{\partial L}{\partial \Pi_A}$. 

\begin{lemma}\label{lem:proj}
Consider a symmetric matrix $A$ and its orthogonal projection operator $\Pi_A$. If $dA$ is a symmetric variation of $A$ then
\begin{equation}
d\Pi_A = 2\left((I-\Pi_A)dA A^+\right)_{sym}\label{seg:proj})
\end{equation}
and 
\begin{equation}
\dfrac{\partial L\circ f}{\partial A} = 2(I-\Pi_A)\left(\dfrac{\partial L}{\partial \Pi_A}\right)_{sym} A^+\label{seg:proj_der}
\end{equation}
\end{lemma}
\begin{proof}
(We drop the subscript of $\Pi_A$ for brevity.) Taking the variation of the basic properties of the projector $\Pi^2=\Pi$ and $\Pi A=A$, we have
\begin{gather}
 d\Pi \Pi + \Pi d\Pi = d\Pi\\
 d\Pi A + \Pi dA = dA
\end{gather}
We then consider the following decomposition of $d\Pi$
\begin{equation*}
 d\Pi = \Pi M\Pi + (I-\Pi)Q\Pi + \Pi Q^\top(I-\Pi) + (I-\Pi)R(I-\Pi)
\end{equation*}
with $M$ and $R$ symmetric, so that $d\Pi$ is symmetric by construction. Plugging into the equations above, we obtain
\begin{gather*}
 2\Pi M\Pi + (I-\Pi)Q\Pi + \Pi Q^\top(I-\Pi) = d\Pi\\
 \Pi M A + (I-\Pi)Q A = (I-\Pi)dA
\end{gather*}
Comparing the first equation with the decomposition of $d\Pi$ above, we infer that $M=R=0$, and so
\begin{gather*}
(I-\Pi)Q\Pi + \Pi Q^\top(I-\Pi) = d\Pi\\
(I-\Pi)Q A = (I-\Pi)dA
\end{gather*}
Multiplying the second equation with $A^+$ at the right hand side gives $(I-\Pi)Q \Pi = (I-\Pi)dA A^+$. Plugging this into the first equation gives the desired result for the variations.

Now we can calculate the partial derivative
\begin{align*}
\dfrac{\partial L}{\partial \Pi_A}:d\Pi_A&=2\left(\dfrac{\partial L}{\partial \Pi_A}\right)_{sym}:((I-\Pi_A)dA A^+) ~~~~~~~~~~~~~~~~~~~~~~~~~~~~~~~~~~~~~~~~~~~~~~~~~~~~~~~~~~~~~~~~~~~~~~~~~~~~~~~~~~~~~~~~~~\text{by }\eqref{seg:proj}\\
  &=2(I-\Pi_A)\left(\dfrac{\partial L}{\partial \Pi_A}\right)_{sym}A^+:dA
\end{align*}
and so \begin{equation}\frac{\partial L\circ f}{\partial A} = 2(I-\Pi_A)\left(\dfrac{\partial L}{\partial \Pi_A}\right)_{sym} A^+\end{equation} 

\end{proof}

The derivation relies only on basic properties of the projector with respect to itself and its matrix: $\Pi_A^2=\Pi_A$ (idempotency of the projector) and $\Pi_A A=A$ (projector leaves the original space unchanged). Note that since $\Pi_A=AA^+$, there exists a non-trivial spectral decomposition in training, although it is `notationally' hidden under $A^+$, which nevertheless requires an SVD computation.

From the perspective of matrix backpropagation we split the computation of $J_1$ into the following 4 layers $F\rightarrow W\rightarrow (M,\Omega)\rightarrow (\Pi_M,\Pi_\Omega)\rightarrow J_1$. We consider them in reverse order from the objective down to the inputs. First the derivative of the Frobenius norm is well known\cite{matrixcookbook} so $\dfrac{\partial J_1}{\partial \Pi_M}=\Pi_M-\Pi_\Omega$ and $\dfrac{\partial J_1}{\partial \Pi_\Omega}=\Pi_\Omega-\Pi_M$. Then we focus on the layer taking as inputs $M$ or $\Omega$ and producing the corresponding projectors \ie $\Pi_M$ and $\Pi_\Omega$. These derivatives are obtained by applying Lemma \ref{lem:proj}. 

Subsequently, we consider the layer receiving $(W,E)$ as inputs and producing $(M,\Omega)$. Under the notation introduced in \S\ref{sec:matrix_backpropagation}, $L=J_1$, $X=(W,E)$ and $Y=f(X)=(M,\Omega)$ as defined above. The following proposition gives the variations of the outputs \ie $\mathcal{L}(dX)=dY=(dM,d\Omega)$ and the partial derivative with respect to the layer $\dfrac{\partial L\circ f}{\partial X}$ as a function of the partial derivatives of the outputs $\dfrac{\partial L}{\partial Y}$\ie $\left(\dfrac{\partial L}{\partial M},\dfrac{\partial L}{\partial \Omega}\right)$.
\begin{prop}
With the notation above, the variations of $M$ and $\Omega$ are 
\begin{equation}
d\Omega=\left(\Omega D^{-1}[dW\one]\right)_{sym}
\end{equation}
and 
\begin{equation}
dM=D^{-1/2} dW D^{-1/2} - \left(M D^{-1}[dW\one]\right)_{sym}
\end{equation}
and the partial derivative of $J_1$ with respect to $W$ is
\begin{align*}
\frac{\partial J_1\circ f}{\partial W} &= D^{-1/2}\frac{\partial J_1}{\partial M} D^{-1/2} + \diag\left(D^{-1}\Omega\left(\frac{\partial J_1}{\partial \Omega}\right)_{sym} - D^{-1}M\left(\frac{\partial J_1}{\partial M}\right)_{sym}\right)\one^\top\label{seg:dj1_dw}
\end{align*}
\end{prop}
\begin{proof}
For a diagonal matrix $D$ under a diagonal variation $dD$, we can show that $d(D^p) = p D^{p-1} dD$ by means of element-wise differentiation. For the particular $D=[W\one]$, we have $dD = [dW\one]$. Using these, we obtain
\begin{multline}
 d\Omega = \frac{1}{2}dD D^{-1/2}EE^\top D^{1/2} + \frac{1}{2}D^{1/2} EE^\top D^{-1/2}dD = \left(D^{1/2} EE^\top D^{-1/2}dD\right)_{sym} = \left(\Omega D^{-1}[dW\one]\right)_{sym}
\end{multline}
and
\begin{align*}
 dM &= -\frac{1}{2}dD D^{-3/2} W D^{-1/2} ~+ D^{-1/2} dW D^{-1/2} - \frac{1}{2}D^{-1/2} W D^{-3/2}dD\\
 &= D^{-1/2} dW D^{-1/2} - \left(M D^{-1}[dW\one]\right)_{sym}
\end{align*}

Then, plugging in the variations we compute the partial derivative
\begin{align*}
 \frac{\partial J_1}{\partial M}:dM + \frac{\partial J_1}{\partial \Omega}:d\Omega&= \left(D^{-1/2}\frac{\partial J_1}{\partial M} D^{-1/2}\right):dW - \left(D^{-1}M\frac{\partial J_1}{\partial M}_{sym}\right):[dW\one] + \left(D^{-1}\Omega\frac{\partial J_1}{\partial \Omega}_{sym}\right):[dW\one]
  \end{align*}
then identifying we obtain
  \begin{align*}
  \frac{\partial J_1}{\partial W} &= D^{-1/2}\frac{\partial J_1}{\partial M} D^{-1/2} + \diag\left(D^{-1}\Omega\frac{\partial J_1}{\partial \Omega}_{sym} - D^{-1}M\frac{\partial J_1}{\partial M}_{sym}\right)\one^\top
\end{align*}

where we used the property $A:[Bx]=A_{ii}(B_{ij}x_j) = (A_{ii}x_j)B_{ij}= \left(\diag(A)x^\top\right):B$. 
\end{proof}

A related optimization objective also presented in \cite{bach06} is
\begin{equation}
 J_2=\frac{1}{2}\left\Vert \Pi_W-\Pi_\Psi\right\Vert_{F}^2,
\end{equation} with $\Psi=E(E^\top E)^{-1}E^\top$. 
Here we consider $\Pi_W=V(V^\top V)^{-1}V^\top$, where $V=D^{1/2}U$. We observe that this is a projector for $W$ by noting that $\Pi_W= D^{1/2}U(U^\top D U)^{-1}U^\top D^{1/2}$ and $M=U\Sigma U^\top=D^{-1/2}WD^{-1/2}$, by eigen decomposition and \eqref{seg:def_M}. Then indeed
\begin{enumerate}
\item Idempotency of $\Pi_W$
\begin{align*}
\Pi_W^2 &= D^{1/2}U(U^\top D U)^{-1}U^\top D U(U^\top D U)^{-1}U^\top D^{1/2} = \Pi_W
\end{align*}
\item $\Pi_W$ leaves $W$ unchanged \begin{align*}\Pi_W W&=\Pi_W (D^{1/2}MD^{1/2})\\&=D^{1/2}U(U^\top D U)^{-1}(U^\top DU)\Sigma U^\top D^{1/2}\\&=D^{-1/2}U\Sigma U^\top D^{-1/2}=W\end{align*}
\end{enumerate}

\begin{prop}
The corresponding partial derivative $\dfrac{\partial J_2}{\partial W}$ is
\begin{equation}
\frac{\partial J_2}{\partial W} = -2(I-\Pi_W)\Pi_\Psi W^+
\end{equation}
\end{prop}
\begin{proof}


Since $\Psi$ does not depend on $W$, then $\dfrac{\partial J_2}{\partial \Psi}=0$, so the derivation is much simpler 
\begin{align}
 \frac{\partial J_2}{\partial W} &= -2(I-\Pi_W)\left(\dfrac{\partial J_2}{\partial \Pi_W}\right) W^+ ~~~~~~~~~~~~~~~~~~~~~~~~~~~~~~~~~~~~~~~~~~\text{by Lemma \ref{lem:proj}}\\
 &=-2(I-\Pi_W)(\Pi_W-\Pi_\Psi) W^+~~~~~~~~~~~~~~~~~~~~~~~~~~~~~~~~~~~~~~~\text{by Frobenius derivative}\\
 &=-2(I-\Pi_W)\Pi_\Psi W^+~~~~~~~~~~~~~~~~~~~~~~~~~~~~~~~~~~~~~~~~~~~~~~~~~~~~~~\text{by idempotency of projector}
\end{align}
\end{proof}

Finally, in both cases, we consider a layer that receives $\Lambda$ and $F$ as inputs and outputs the data similarity $W=F\Lambda F^\top$. Following the procedure of section \ref{sec:matrix_backpropagation}, first we compute the first order variations $dW=dF \Lambda F^\top + F d\Lambda F^\top + F \Lambda dF^\top$. We then use the trace properties to make the partial derivatives identifiable
\begin{align*}
 dJ_i = \frac{\partial J_i}{\partial W}:dW &= \frac{\partial J_i}{\partial W}:\left(F d\Lambda F^\top\right) + 2\frac{\partial J_i}{\partial W}:(dF\Lambda F^\top )_{sym}\\
 &= \left(F^\top\frac{\partial J_i}{\partial W} F\right):d\Lambda + 2 \left(\frac{\partial J_i}{\partial W}_{sym}F\Lambda^\top\right):dF
\end{align*}
Thus we obtain
\begin{equation}\frac{\partial J_i}{\partial \Lambda} = F^\top\frac{\partial J_i}{\partial W} F\end{equation} and \begin{equation}{\frac{\partial J_i}{\partial F} = 2\left(\frac{\partial J_i}{\partial W}\right)_{sym}F\Lambda^\top}\, \end{equation}

Note that when $J=J_2$ then $\dfrac{\partial J_2}{\partial\Lambda}=0$, since $(I-\Pi_W)F=F^\top(I-\Pi_W)=0$. Thus we cannot learn $\Lambda$ by relying on our projector trick, but there is no problem learning $F$, which is our objective, and arguably more interesting, anyway.

\FIX{An important feature of} our formulation is that we do not restrict the rank in training. During alignment, the optimization may choose to collapse certain directions thus reducing rank. We \FIX{prove} a topological lemma implying that if the Frobenius distance between the projectors (such as in the two objectives $J_1$, $J_2$) drops below a certain value, then the ranks of the two projectors will match. Conversely, if for some reason the ranks cannot converge, the objectives are bounded away from zero. 
The following lemma shows that when the projectors of two matrices $A$ and $B$ are close enough in the $\lVert\cdot\rVert_2$ norm, then the matrices have the same rank.

\begin{lemma} Let $A,B\in\mathbb{R}^{m\times n}$, and $\Pi_A$, $\Pi_B$ their respective orthogonal projectors. If $\lVert\Pi_A - \Pi_B\rVert_2 < 1$ then $\rank A = \rank B$.
\end{lemma}
\begin{proof} The spectral norm $\lVert\cdot\rVert_2$ can indeed be defined as $\lVert A\rVert_2 = \sup_{\lVert x\rVert_2\neq 0}\frac{\lVert Ax\rVert}{\lVert x\rVert}$. We assume that the ranks of $A$ and $B$ are different, \ie w.l.o.g.~$\rank A>\rank B$. By the fundamental theorem of linear algebra there exists a vector $v$ in the range of $A$ (so that $\Pi_A v = v$), that is orthogonal to the range of $B$ (so that $\Pi_B v =0$). We have then
\[\lVert\Pi_A - \Pi_B\rVert_2 \geq \frac{\lVert \Pi_A v - \Pi_B v\rVert}{\lVert v\rVert} = \frac{\lVert \Pi_A v\rVert}{\lVert v\rVert} = 1 \]
which is a contradiction. 
\end{proof}

Given that the Frobenius norm controls the spectral norm, \ie~$\lVert A\rVert_2 \leq \lVert A\rVert_F$ (\S2.3.2 of \cite{Golub96}), an immediate corollary is that when $J_2$ is bounded above by $1/2$, then $||A||_2 < 1$ and the spaces spanned by $W$ and $EE^\top$ are perfectly aligned, \ie
\begin{equation}
 J_2(W) < \frac{1}{2} \Rightarrow \rank(W) = \rank(EE^\top)
\end{equation}

\section{Experiments}\label{sec:exps}
\FIX{In this section we validate the proposed methodology by constructing models on standard datasets for region-based object classification, like Microsoft COCO \cite{mscoco}, and for image segmentation on BSDS \cite{amfm_pami2011}. 
A matconvnet \cite{matconvnet} implementation of our models and methods is publicly available. }

\subsection{Region Classification on MSCOCO}

For recognition we use the MSCOCO dataset \cite{mscoco}, which provides 880k segmented training instances across 80 classes, divided into training and validation sets. The main goal is to assess our second-order pooling layer in various training settings. A secondary goal is to study the behavior of ConvNets learned from scratch on segmented training data. This has not been explored before in the context of deep learning because of the relatively small size of the datasets with associated object segmentations, such as PASCAL VOC \cite{everingham2010pascal}.

\FIX{The experiments in this section use the convolutional architecture component of AlexNet\cite{krizhevsky2012imagenet} with the global O$_2$P layers we propose in order to obtain DeepO$_2$P  models with both classification and fully connected (FC) layers in the same topology as Alexnet.} We crop and resize each object bounding box to have 200 pixels on the largest side, then pad it to the standard AlexNet input size of 227x227 with small translation jittering, to limit over-fitting. We also randomly flip the images in each mini-batch horizontally, as in standard practice. Training is performed with stochastic gradient descent with momentum. We use the same batch size (100 images) for all methods but the learning rate was optimized for each model independently. 
All the DeepO$_2$P models used the same $\epsilon=10^{-3}$ parameter value in \eqref{eqn:deepo2p_svd}. \\

\begin{table*}[!htbp]
\centering
\begin{tabular}{|c|c|c|c|c|c|c|c|}
\hline
 Method  & SIFT-O$_2$P  & AlexNet & AlexNet (S) & DeepO$_2$P& DeepO$_2$P(S) & DeepO$_2$P-FC&  DeepO$_2$P-FC(S)\\
 \hline
 Results & 36.4          & 25.3    & 27.2       &  28.6      & 32.4          &   \textbf{25.2}   &    28.9   \\
 \hline
\end{tabular}\vspace{2mm}
\caption{Classification error on the validation set of MSCOCO (lower is better). Models with (S) suffixes were trained from scratch (\ie~random initialization) on the MSCOCO dataset. The DeepO$_2$P models only use a classification layer on top of the DeepO$_2$P layer whereas the DeepO$_2$P-FC also have fully connected layers in the same topology as AlexNet. All parameters of our proposed global models are refined jointly, end-to-end, using the proposed matrix backpropagation.
}
\label{tab:recog}
\end{table*}

\noindent{\bf Architecture and Implementation details.} Implementing the spectral layers efficiently is challenging since the GPU support for SVD is still very limited and our parallelization efforts even \FIX{using} the latest CUDA 7.0 solver API have delivered a slower implementation than the \FIX{standard CPU-based. Consequently,} we use CPU implementations and incur a penalty for moving data back and forth to the CPU. The numerical experiments revealed that an implementation in single precision obtained a significantly less accurate gradient for learning. Therefore all computations in our proposed layers, both in the forward and backward passes, are made in double precision. In experiments we still noticed a significant accuracy penalty \FIX{due to inferior} precision in all the other layers (above and below the structured ones), still computed in single precision, on the GPU. 

\begin{table*}[!htbp]
\centering
\begin{tabular}{|c||c||c|c||c|c||c|c|}
\hline
Arch    &   \cite{cour2005spectral}                 & \multicolumn{2}{c||}{AlexNet}            &\multicolumn{4}{c|}{VGG} \\
\cline{5-8}
Layer   &  &   \multicolumn{2}{c||}{ReLU-5}         &   \multicolumn{2}{c||}{ReLU-4} &   \multicolumn{2}{c|}{ReLU-5} \\
\hline
Method  &  NCuts            & NCuts   & \FIX{DeepNCuts}   &NCuts        &DeepNCuts &NCuts &DeepNCuts\\
\hline
Results & .55 (.44)           &     .59 (.49) &    .65 (.56)     & .65 (.56)     & .73 (.62) &  .70 (.58)  & .74 (.63)\\
\hline
\end{tabular}
\vspace{2mm}
\caption{Segmentation results give best and average covering to the pool of ground truth segmentations on the BSDS300 dataset \cite{amfm_pami2011} (larger is better). We use as baselines the original normalized cuts \cite{cour2005spectral} using intervening contour affinities as well as normalized cuts with affinities derived from non-finetuned deep features in different layers of AlexNet (ReLU-5 - the last local ReLU before the fully connected layers) and VGG (first layer in block 4 and the last one in block 5). Our DeepNCuts models are trained end-to-end, based on the proposed matrix backpropagation methodology, using the objective $J_2$. }\label{tab:seg_res}
\end{table*}

The second formal derivation of the non-linear spectral layer based on an eigen-decomposition of $Z=F^\top F+\epsilon I$ instead of SVD of $F$ is also possible but our numerical experiments favor the formulation using SVD. The alternative implementation, which is formally correct, exhibits numerical instability in the derivative when multiple eigenvalues have very close values, thus producing blow up in $\tilde{K}$. Such numerical issues are expected to appear under some implementations, when complex layers like the ones presented here are integrated in deep network settings.\\

\noindent{\bf Results.} The results of the recognition experiment are presented in table \ref{tab:recog}. They show that our proposed DeepO$_2$P-FC models, containing global layers, outperform standard convolutional pipelines based on AlexNet, on this problem. The bottom layers are pre-trained on ImageNet using AlexNet, and this might not provide the ideal initial input features. However, despite this potentially unfavorable initialization, our model jointly refines all parameters (both convolutional, and corresponding to global layers), jointly, end to end, using a consistent cost function.

\FIX{We note that the fully connected layers on top of the DeepO$_2$P layer offer good performance benefits. O$_2$P over hand-crafted SIFT performs considerably less well than our DeepO$_2$P models, suggesting that large potential gains can be achieved when deep features replace existing descriptors.}

\subsection{Full-Image Segmentation on BSDS300}

We use the BSDS300 dataset to validate our deep normalized cuts approach. BSDS contains 200 training images and 100 testing images and human annotations of all the relevant regions in the image. Although small by the standards of neural network learning it provides exactly the supervision we need to refine our model using global information. Note that since the supervision is pixel-wise, the number of effective datapoint constraints is much larger. We evaluate using the average and best covering metric under the Optimal Image Scale (OIS) criterion \cite{amfm_pami2011}. Given a set of full image segmentations computed for an image, selecting the one that maximizes the average and best covering, respectively, compared to the pool of ground truth segmentations.\\

\noindent{\bf Architecture and Implementation details.} 
We use both the AlexNet\cite{krizhevsky2012imagenet} and the VGG-16\cite{Simonyan14c} architectures to feed our global layers. All the parameters of the deep global models (including the low-level features, pretrained on ImageNet) are refined end-to-end. We use a linear affinity but we need all entries of $W$ to be positive. Thus, we use ReLU layers to feed the segmentation ones. Initially, we just cascaded our segmentation layer to different layers in AlexNet but the resulting models were hard to learn. Our best results were obtained by adding two Conv-ReLU pairs initialized randomly before the normalized cuts layer. This results in many filters in the lower layer (256 for AlexNet and 1024 for VGG) for high capacity but few in the top layer (20 dimensions) to limit the maximal rank of $W$. For AlexNet we chose the last convolutional layer while for VGG we used both the first ReLU layer in block\footnote{We call a block the set of layers between two pooling levels.} 4 and the top layer from block 5. This gives us feeds from layers with different invariances, receptive field sizes (32 vs. 132 pixels)  and coarseness (block 4 has $2\times$ the resolution of 5). We used an initial learning rate of $10^{-4}$ but $10\times$ larger rates for the newly initialized layers. A dropout layer between the last two layers with a rate of $.25$ reduces overfitting. In inference, we generate 8 segmentations by clustering\cite{bach06} then connected components are split into separate segments.\\

\noindent{\bf Results.} The results in table \ref{tab:seg_res} show that in all cases we obtain important performance improvements with respect to the corresponding models that perform inference directly on original AlexNet/VGG features.  Training using our Matlab implementation  takes ~2 images/s considering 1 image per batch while testing at about ~3 images/s on a standard Titan Z GPU with an 8 core E5506 CPU. In experiments we monitor both the objective and the rank of the similarity matrix. Rank reduction is usually a good indicator of performance in both training and testing. In the context of the rank analysis in \S\ref{sec:scl}, we interpret these findings to mean that if the rank of the similarity is too large compared to the target, the objective is not sufficient to lead to rank reduction. However if the rank of the predicted similarity and the ground truth are initially not too far apart, then rank reduction (although not always rank matching) does occur and improves the results.

\begin{figure*}[htbp]
\begin{center}
\[\begin{array}{|C|C|C|C|C|C|C|C|}
\hline
& \multicolumn{2}{c|}{ \text{AlexNet}}                 &\multicolumn{4}{c|}{  \text{VGG}}                                          &\text{Ground}\\
\cline{4-7}
{ Image} &   \multicolumn{2}{c|}{ \text{ReLU-5}}         &   \multicolumn{2}{c|}{\text{ReLU-4}} &   \multicolumn{2}{c|}{\text{ReLU-5}} &\multicolumn{1}{c|}{\text{ Truth}}\\
\cline{2-7}
&{\small NCuts}   & {\small DeepNCuts}   &{\small NCuts}        &{\small DeepNCuts} &{\small NCuts} &{\small DeepNCuts}&\text{\small (Human)}\\
\hline
\end{array}\]\vspace{-4.5mm}
\includegraphics[width=1\textwidth]{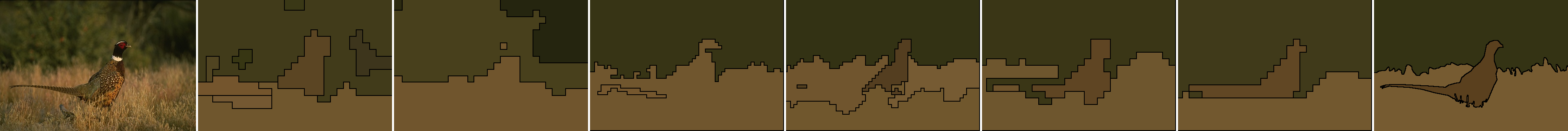}
\includegraphics[width=1\textwidth]{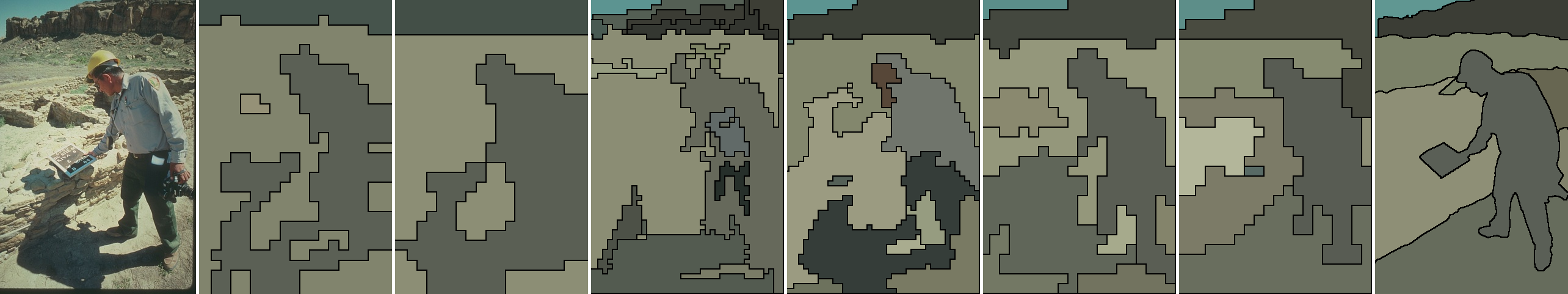}
\includegraphics[width=1\textwidth]{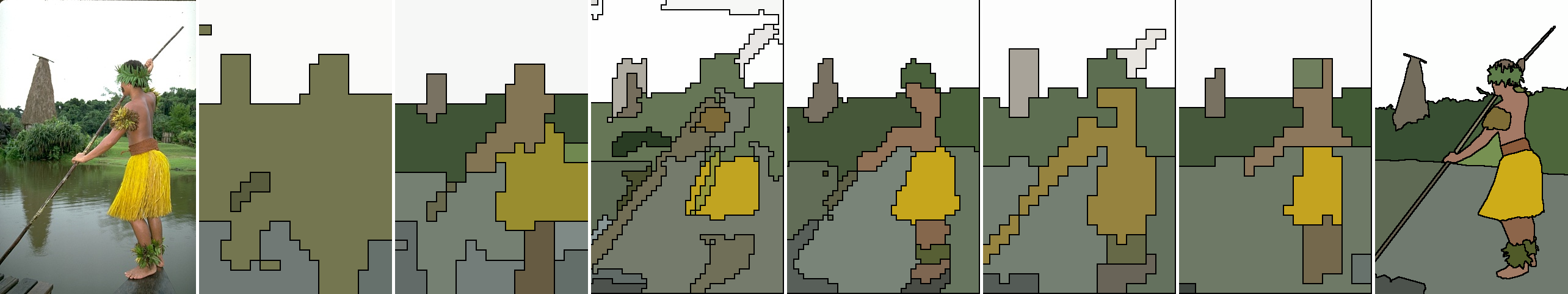}
\includegraphics[width=1\textwidth]{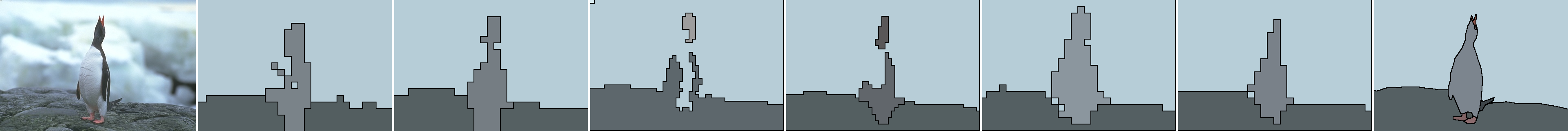}
\includegraphics[width=1\textwidth]{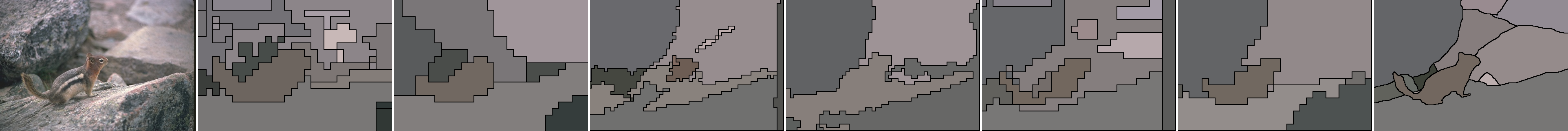}
\includegraphics[width=1\textwidth]{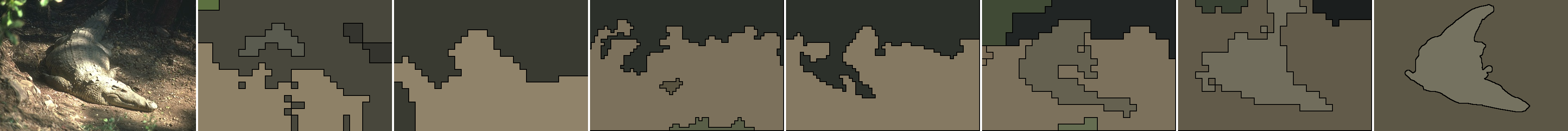}
\includegraphics[width=1\textwidth]{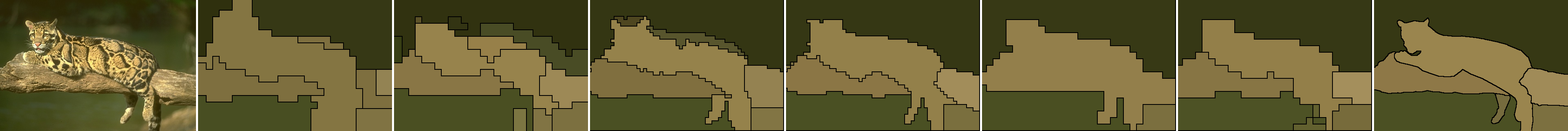}
\includegraphics[width=1\textwidth]{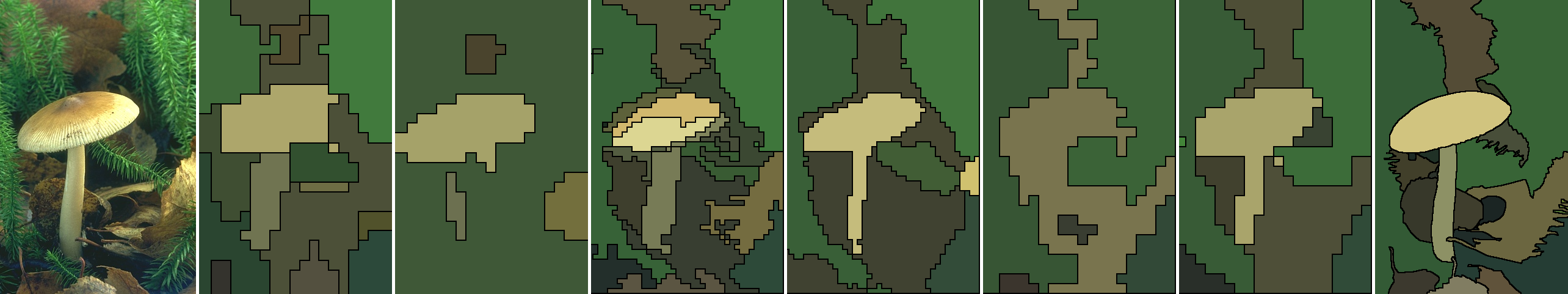}
\includegraphics[width=1\textwidth]{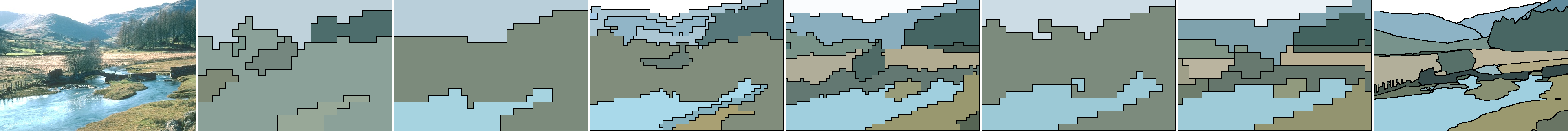}
\includegraphics[width=1\textwidth]{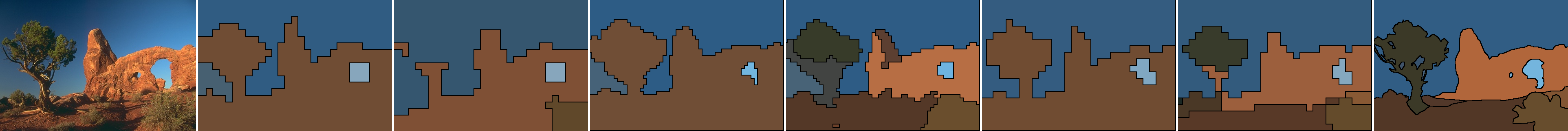}
\includegraphics[width=1\textwidth]{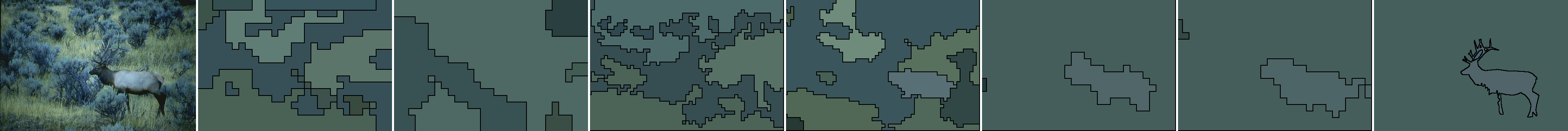}
\includegraphics[width=1\textwidth]{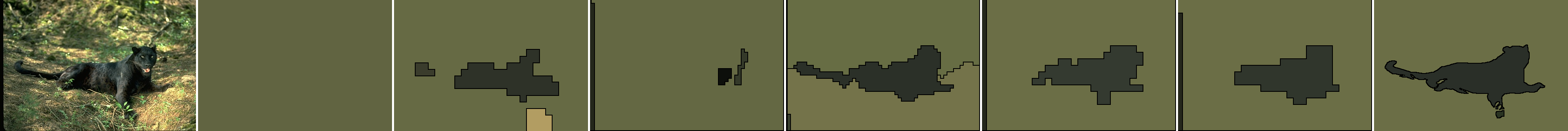}
\caption{Segmentation results on images from the test set of BSDS300. We show on the first column the input image followed by a baseline (original parameters) and our DeepNcuts both using AlexNet ReLU-5. Two other pairs of baselines and DeepNCut models trained based on the $J_2$ objective follow. The first pair uses ReLU-4 and the second ReLU-5. The improvements obtained by learning are both quantitatively significant and easily visible on this side-by-side comparison.}
\end{center}
\end{figure*}

\section{Conclusions}

Motivated by the recent success of deep network architectures, in this work we have introduced the mathematical theory and the computational blocks that support the development of more complex models with layers that perform structured, global computations like segmentation or higher-order pooling. Central to our methodology is the development of the matrix backpropagation methodology which relies on the calculus of \FIX{adjoint} matrix variations. We provide detailed derivations, operating conditions for spectral and non-linear layers, and illustrate the methodology for normalized cuts and second-order pooling layers. Our region visual recognition and segmentation experiments based on MSCoco and BSDS show that deep networks relying on second-order pooling and normalized cuts layers, trained end-to-end using the introduced practice of matrix backpropagation, outperform counterparts that do not take advantage of such global layers.

\paragraph{Acknowledgements.} This work was partly supported by CNCS-UEFISCDI under CT-ERC-2012-1, PCE-2011-3-0438, JRP-RO-FR-2014-16. We thank J. Carreira for helpful discussions and Nvidia for a generous graphics board donation.

\section*{Appendix}
\subsection{Notation and Basic identities}
In this section we present for completeness the notation and some basic linear algebra identities that are useful in the calculations associated to matrix backpropagation and its instantiation for log-covariance descriptors \cite{carreira2014free,Arsigny07geometric} and normalized cuts segmentation \cite{shi00ncuts}.

The following notation is used in the derivations
\begin{itemize}
 \item The symmetric part $A_{sym}=\frac{1}{2}(A^\top+A)$ of a square matrix $A$.
 \item The diagonal operator $A_{diag}$ for an arbitrary matrix $A\in\mathbb{R}^{m\times n}$, which is the $m\times n$ matrix which matches $A$ on the main diagonal and is 0 elsewhere. Using the notations $\diag(A)$ and $[x]$ to denote the diagonal of $A$ (taken as a vector) and the diagonal matrix with the vector $x$ in the diagonal resp., then $A_{diag}=[\diag(A)]$.
 \item The colon-product $A:B = \displaystyle\sum_{i,j}A_{ij}B_{ij}=\Tr(A^\top B)$ for matrices $A,B\in\mathbb{R}^{m\times n}$, and the associated Frobenius norm $\lVert A\rVert:=\sqrt{A:A}$.
 \item The Hadamard (element-wise) product $A\circ B$.
\end{itemize}

We note the following properties of the matrix inner product ``:'' :
\begin{align}
 \label{col:trans}&A:B = A^\top:B^\top = B:A\\
 \label{col:prod}&A:(BC) = (B^\top A):C = (AC^\top):B\\
 \label{col:diag}&A:B_{diag} = A_{diag}:B\\
 \label{col:sym}&A:B_{sym} = A_{sym}:B\\
 \label{col:circ}&A:(B\circ C) = (B\circ A):C\\
 \label{col:block}& \bigl(A_1|A_2\bigr):\bigl(B_1|B_2\bigr)=A_1:B_1 + A_2:B_2
\end{align}

{\small
\bibliographystyle{ieeetr}
\bibliography{snn_arxiv_v2}
}

\end{document}